\def\supp{1} 

\documentclass{article}
\pdfoutput=1
\usepackage{fullpage}
\usepackage{amstext,amssymb,amsmath}
\usepackage{amsthm}
\usepackage{mathtools}
\usepackage{verbatim}
\usepackage{paralist}
\usepackage{dsfont}
\usepackage{hyperref}
\usepackage[noend]{algorithmic}
\usepackage{algorithm}

\newcommand{\nuc}[1]{\left\|#1\right\|_{\sf nuc}}

\newcommand{\frob}[1]{\left\|#1\right\|_{F}}

\newtheorem{theorem}{Theorem}[section]
\newtheorem{infTheorem}{Informal Theorem}[section]
\newtheorem{corollary}{Corollary}[section]
\newtheorem{lemma}[theorem]{Lemma}
\newtheorem{definition}[theorem]{Definition}

\newcommand{\ip}[2]{\left\langle #1,#2\right\rangle}

\newcommand{\lfrob}[1]{\left\|#1\right\|_F}

\renewcommand{\paragraph}[1]{\vspace{3pt}\noindent\textbf{#1}}

\newcommand{\ltwo}[1]{\left\|#1\right\|_2}

\newcommand{\eps}{\epsilon}
\newcommand{\A}{\mathcal{A}}

\newcommand{\I}{\mathbb{I}}

\newcommand{\E}{\mathbb{E}}

\newcommand{\re}{\Re}

\newcommand{\grad}{\bigtriangledown}

\newcommand{\mypar}[1]{\noindent{\bf\em {#1}:}}

 \newcommand{\mat}[1]{{#1}\,}
\newcommand{\ktnote}[1]{}

\newcommand{\ignore}[1]{}

\DeclareMathOperator*{\argmin}{{\sf argmin}}

\theoremstyle{remark}
\newtheorem{remark}{Remark}

\newcommand{\pomega}{{\sf P}_\Omega}
\newcommand{\matT}[2]{{{#1}^{(#2)}}}

\newcommand{\halpha}{{\widehat \alpha}}

\newcommand{\jc}{joint differential privacy\,}

\newcommand{\Jc}{Joint differential privacy\,}
\newcommand{\rank}{{\sf rank}\,}
\newcommand{\Ypriv}{Y}

\newcommand{\Yp}{{\widehat{Y}}}

\newcommand{\hY}{{\widehat{\mat{Y}}}}
\newcommand\numberthis{\addtocounter{equation}{1}\tag{\theequation}}
\newcommand{\vp}{{\widehat{V}}}
\newcommand{\Pip}{{\widehat{\Pi}}}

\newcommand{\om}[1]{\ifnum\supp=1\textcolor{green}{[Om: #1]}\else\ignorespaces\fi}
\newcommand{\todo}[1]{\ifnum\supp=1\textcolor{green}{[TODO: #1]}\else\ignorespaces\fi}
\newcommand{\pj}[1]{\ifnum\supp=1\textcolor{red}{[PJ: #1]}\else\ignorespaces\fi}
\newcommand{\at}[1]{\ifnum\supp=1\textcolor{blue}{[AT: #1]}\else\ignorespaces\fi}
\makeatletter
\def\blfootnote{\gdef\@thefnmark{}\@footnotetext}
\makeatother

\begin{document}
\title{Differentially Private Matrix Completion Revisited\blfootnote{Accepted for presentation at International Conference on Machine Learning (ICML) 2018.}}

\author{Prateek Jain\thanks{Microsoft Research. Email:
    \texttt{prajain@microsoft.com}.} \and Om Thakkar\thanks{Department of Computer Science,
   Boston University. Email:
    \texttt{omthkkr@bu.edu}.}  \and Abhradeep Thakurta\thanks{Computer Science Department,
    University of California Santa Cruz. Email:
    \texttt{aguhatha@ucsc.edu}.} }
\maketitle

\begin{abstract}

We provide the first provably joint differentially private algorithm with formal utility guarantees for the problem of \emph{user-level} privacy-preserving collaborative filtering. Our algorithm is based on the Frank-Wolfe method, and it consistently estimates the underlying preference matrix as long as the number of users $m$ is $\omega(n^{5/4})$, where $n$ is the number of items, and each user provides her preference for at least $\sqrt{n}$ randomly selected items. Along the way, we provide an optimal differentially private algorithm for singular vector computation, based on the celebrated Oja's method, that provides significant savings in terms of space and time while operating on sparse matrices. We also empirically evaluate our algorithm on a suite of datasets, and show that it consistently outperforms the state-of-the-art {\em private} algorithms. 
\end{abstract}

\section{Introduction}
\label{sec:intro}

Collaborative filtering (or matrix completion) is a popular approach for modeling the recommendation system problem, where the goal is to provide personalized recommendations about certain items to a user \cite{KorenB15}. In other words, the objective of a personalized recommendation system is to learn the entire users-items preference matrix $Y^*\in \Re^{m\times n}$ using a small number of user-item preferences $Y^*_{ij}, (i,j)\in [m]\times [n]$, where $m$ is the number of users and $n$ is the number of items.  Naturally, in absence of any structure in $Y^*$, the problem is ill-defined as the unknown entries of $Y^*$ can be arbitrary. Hence, a popular modeling hypothesis is that the underlying preference matrix $Y^*$ is low-rank, and thus, the collaborative filtering problem reduces to that of low-rank matrix completion \cite{recht2011simpler, candes2012exact}.  One can also enhance this formulation using side-information like user-features or item-features  \cite{inductivemc}.

Naturally, personalization problems require collecting and analyzing sensitive customer data like their preferences for various items, which can lead to serious privacy breaches \cite{korolova2010privacy,NS10,CalandrinoKNFS11}. In this work, we attempt to address this problem of privacy-preserving recommendations using collaborative filtering \cite{MM09,liu2015fast}. We  answer the following question in the {\bf affirmative}: \emph{Can we design a matrix completion algorithm which keeps {\bf all} the ratings of a user private, i.e., guarantees {\bf user-level privacy} while still providing accurate recommendations?} In particular, we provide the \emph{first} differentially private \cite{DMNS} matrix completion algorithms with {\em provable} accuracy guarantees. Differential privacy (DP) is a rigorous privacy notion which formally protects the privacy of any user participating in a statistical computation by controlling her influence to the final output.  

Most of the prior works on DP matrix completion (and low-rank approximation) \cite{BDMN05,ChanSS11,HR12,HR13,KapralovT13,dwork2014analyze} have provided guarantees which are non-trivial only in the {\em entry-level} privacy setting, i.e., they preserve privacy of only a single rating of a user. Hence, they are not suitable for preserving a user's privacy in practical recommendation systems. In fact, their trivial extension to user-level privacy leads to vacuous bounds (see Table~\ref{tab:bounds}).  Some works \cite{MM09,liu2015fast} do serve as an exception, and directly address the user-level privacy problem. However, they only show empirical evidences of their effectiveness; they do not provide formal error bounds.\footnote{In case of \cite{liu2015fast}, the DP guarantee itself might require an exponential amount of computation.} In contrast, we provide an efficient algorithm based on the classic Frank-Wolfe (FW) procedure \cite{FW56}, and show that it gives strong utility guarantees while preserving user-level privacy. Furthermore, we empirically demonstrate its effectiveness on various benchmark datasets.

Our private FW procedure needs to compute the top right singular vector of a sparse user preference matrix, while preserving DP. For practical recommendation systems with a large number of items, this step turns out to be a significant bottleneck both in terms of space as well as time complexity. To alleviate this issue, we provide a method, based on the celebrated Oja's algorithm \cite{JJKNS16}, which is nearly optimal in terms of the accuracy of the computed singular vector while still providing significant improvement in terms of space and computation. In fact, our method can be used to speed-up even the vanilla differentially private PCA computation \cite{DTTZ13}. To the best of our knowledge, this is the first algorithm for DP singular value computation with optimal utility guarantee, that also exploits the sparsity of the underlying matrix.

\mypar{Notion of privacy} To measure privacy, we select \emph{differential privacy}, which is a de-facto privacy notion for large-scale learning systems, and has been widely adopted by the academic community as well as big corporations like Google \cite{erlingsson2014rappor}, Apple \cite{apple}, etc. The underlying principle of \emph{standard DP} is that the output of the algorithm should not change significantly due to  presence or absence of any user. In the context of matrix completion, where the goal is to release the {\em entire preference matrix} while preserving privacy, this implies  that the computed ratings/preferences for any particular user cannot depend strongly on {\em her own personal preferences}. Naturally, the resulting preference computation is going to be trivial and inaccurate (which also follows from the reconstruction attacks of \cite{DN03} and \cite{HR12}).

To alleviate this concern, we consider a relaxed but natural DP notion (for recommendation systems) called \emph{\jc} \cite{kearns2014mechanism}. Consider an algorithm $\mathcal{A}$ that produces individual outputs $Y_i$ for each user $i$, i.e., the $i$-th row of preference matrix $Y$. Joint DP  ensures that for each user $i$, the output of $\mathcal{A}$ for all other users (denoted by $Y_{-i}$) does not reveal ``much'' about the preferences of user $i$. That is, the recommendations made to all the users except the $i$-th user do not depend significantly upon the $i$-th user's preferences. Although not mentioned explicitly,  previous works on DP matrix completion \cite{MM09,liu2015fast}  strive to ensure Joint DP. Formal definitions are provided in Section \ref{sec:privDef}.

\mypar{Granularity of privacy} DP protects the information about a user in the context of presence or absence of her data record. Prior works on DP matrix completion \cite{MM09,liu2015fast}, and its close analogue, low-rank approximation \cite{BDMN05,ChanSS11,HR12,DTTZ13, HR13}, have considered different variants of the notion of a data record. Some have considered a single entry in the matrix $Y^*$ as a data record (resulting in \emph{entry-level privacy}), whereas others have considered a more practical setting where the complete row is a data record (resulting in \emph{user-level privacy}). In this work, we present all our results in the strictly harder user-level privacy setting. To ensure a fair comparison, we present the results of prior works in the same setting.

\subsection{Problem definition: Matrix completion}
\label{sec:matrixCompletion}
The goal of a low-rank matrix completion problem is to estimate a low-rank (or a convex relaxation of bounded nuclear norm) matrix $Y^*\in \re^{m\times n}$, having seen only a small number of entries from it. Here, $m$ is the number of users, and $n$ is the number of items. Let $\Omega=\{(i,j)\subseteq [m]\times [n]\}$ be the index set of the observed entries from $\mat{Y}^*$, and let $\pomega: \re^{m\times n} \rightarrow \re^{m\times n}$ be a matrix  operator s.t. $\pomega(Y)_{ij}=Y_{ij}$ if $(i,j)\in \Omega$, and $0$ otherwise. Given, $\pomega(Y^*)$, the objective is to output a matrix $Y$ such that the following generalization error, i.e., the error in approximating a uniformly random entry from the matrix $Y^*$, is minimized:
\begin{equation}
F(Y)=\E_{(i,j)\sim_{\sf unif}[m] \times [n]}\left[\left(Y_{ij}-Y^*_{ij}\right)^2\right].
\label{eqn: gen_error_def}
\end{equation}
Generalization error captures the ability of an algorithm to predict unseen samples from $Y^*$. We would want the generalization error to be $o(1)$ in terms of the problem parameters when $\Omega = o(mn)$. Throughout the paper, we will assume that $m > n$. 

\subsubsection{Our contributions}
\label{sec:contrib}
In this work, we provide the first joint DP algorithm for low-rank matrix completion with formal non-trivial error bounds, which are summarized in Tables \ref{tab:bounds} and \ref{tab:bounds1}. At a high level, our key result can be  summarized as follows: 
\begin{infTheorem}[Corresponds to Corollary \ref{cor:abcfd12}]
	Assume that each entry of a hidden matrix $Y^*\in\re^{m\times n}$ is in $[-1,1]$, and there are $\sqrt{n}$ observed entries per user. Also, assume that the nuclear norm of $Y^*$ is bounded by $O(\sqrt{mn})$, i.e., $Y^*$ has nearly constant rank.
	Then, there exist $(\epsilon,\delta)$-joint differentially private algorithms that have $o(1)$ generalization error as long as $m=\omega(n^{5/4})$.
	\label{thm:kjnjkfjkf132}
\end{infTheorem}
In other words, even with $\sqrt{n}$ observed ratings per user, we obtain asymptotically the correct estimation of each entry of $Y^*$  on average, as long as $m$ is large enough. The sample complexity bound dependence on $m$ can be strengthened by making additional assumptions, such as \emph{incoherence}, on $Y^*$.
\ifnum \supp=0
See the supplementary material for details.
\else
See Appendix~\ref{app:SVDapprox} for details.
\fi 

Our algorithm is based on two important ideas: a) using local and global computation, b) using the Frank-Wolfe method as a base optimization technique.

\mypar{Local and global computation} The key idea that defines our algorithm, and allows us to get strong error bounds under joint DP is splitting the algorithm into two components: \emph{global} and \emph{local}. Recall that each row of the hidden matrix $Y^*$ belongs to an individual user. The global component of our algorithm computes statistics that are aggregate in nature (e.g., computing the correlation across columns of the revealed matrix $\pomega(Y^*)$). On the other hand, the local component independently fine-tunes the statistics computed by the global component to generate accurate predictions for each user. Since the global component depends on the data of all users, adding noise to it (for privacy) does not significantly affect the accuracy of the predictions. \cite{MM09,liu2015fast} also exploit a similar idea of segregating the computation, but they do not utilize it formally to provide non-trivial error bounds.

\mypar{Frank-Wolfe based method} We use the standard nuclear norm formulation \cite{recht2011simpler,shalev2011large, tewari2011greedy, candes2012exact} for the matrix completion problem: {\small \begin{equation}
\label{eq:nucnorm}
\min\limits_{\nuc{Y}\leq k} \widehat{F}(Y),\end{equation}}
	where $\widehat{F}(Y) = \frac{1}{2|\Omega|}\|\pomega(Y-Y^*)\|_F^2$,  $\nuc{Y}$ is the sum of singular values of $Y$, and the underlying hidden matrix $Y^*$ is assumed to have nuclear norm of at most $k$. Note that we denote the empirical risk of a solution $Y$ by $\widehat{F}(Y)$ throughout the paper. We use the popular Frank-Wolfe algorithm \cite{FW56,jaggi2010simple} as our algorithmic building block. At a high-level, FW computes the solution to \eqref{eq:nucnorm} as a convex combination of rank-one matrices, each with nuclear norm at most $k$. These matrices are added iteratively to the solution.
	
Our main contribution is to design a version of the FW method that preserves Joint DP. That is, if the standard FW algorithm decides to add matrix $u\cdot v^T$ during an iteration, our private FW computes a noisy version of $v\in \re^n$ via its global component. Then, each user computes the respective element of $u\in \Re^m$ to obtain her update. The noisy version of $v$ suffices for the Joint DP guarantee, and allows us to provide the strong error bound in Theorem \ref{thm:kjnjkfjkf132} above. 

We want to emphasize that the choice of FW as the underlying matrix completion algorithm is critical for our system. FW updates via rank-one matrices in each step. Hence, the error due to noise addition in each step is small (i.e., proportional to the rank), and allows for an easy decomposition into the local-global computation model. Other standard techniques like proximal gradient descent based techniques \cite{svt, ialm} can involve nearly {\em full-rank} updates in an iteration, and hence might incur large error, leading to arbitrary inaccurate solutions. Note that though a prior work \cite{TTZ15} has proposed a DP Frank-Wolfe algorithm for high-dimensional regression, it was for a completely different problem in a different setting where the segregation of computation into global and local components was not necessary.

\mypar{Private singular vector of sparse matrices using Oja's method} 
Our private FW requires computing a noisy covariance matrix which implies $\Omega(n^2)$ space/time complexity for $n$ items. Naturally, such an algorithm does not scale to practical recommendation systems. In fact, this drawback exists even for standard private PCA techniques \cite{DTTZ13}. Using insights from the popular Oja's method, we provide a technique (see Algorithm~\ref{Alg:Eigs}) that has a linear dependency on $n$ as long as the number of ratings per user is small. Moreover, the performance of our private FW method isn't affected by using this technique. 

\mypar{SVD-based method} 
\ifnum \supp=0
In the supplementary material,
\else
In Appendix~\ref{app:SVDapprox},
\fi
 we also extend our technique to a singular value decomposition (SVD) based method for matrix completion/factorization. Our utility analysis shows that there are settings where this method outperforms our FW-based method, but in general it can provide a significantly worse solution. The main goal is to study the power of the simple SVD-based method, which is still a popular method for collaborative filtering.

 \mypar{Empirical results} Finally, we show that along with providing strong analytical guarantees, our private FW also performs well empirically. In particular, we show its efficacy on benchmark collaborative filtering datasets like Jester \cite{Jester}, MovieLens \cite{Harper2}, the Netflix prize dataset \cite{Netf}, and the Yahoo! Music recommender dataset \cite{YahooMov}. Our algorithm consistently outperforms (in terms of accuracy) the existing state-of-the-art DP matrix completion methods (SVD-based method by \cite{MM09}, and a variant of projected gradient descent \cite{cai2010singular,BassilyST14,DPDL}). 

\subsection{Comparison to prior work}
\label{sec:compPre}
As discussed earlier, our results are the first to provide non-trivial error bounds for DP matrix completion. For comparing different results, we consider the following setting of the hidden matrix $Y^*\in\re^{m\times n}$ and the set of released entries $\Omega$: i) $|\Omega|\approx m\sqrt n$, ii) each row of $Y^*$ has an $\ell_2$ norm of $\sqrt n$, and iii) each row of $\pomega(Y^*)$ has $\ell_2$-norm at most $n^{1/4}$, \emph{i.e.,} $\approx \sqrt{n}$ random entries are revealed for each row. Furthermore, we assume the spectral norm of $Y^*$ is at most $O(\sqrt{mn})$, and $Y^*$ is rank-one. Note that these conditions are satisfied by a matrix $Y^*=u\cdot v^T$ where $u_i, v_j \in [-1, 1]$ $\forall i,j$, and $\sqrt{n}$ random entries are observed {\em per user}. 

\begin{table}[h]
\parbox{.42\linewidth}{
\begin{center} 
\begin{tabular}{|c|c|c|}
\hline
{\bf Algorithm}  &{\bf Bound} &{\bf Bound}\\
 &{\bf on $m$} &{\bf on $|\Omega|$}\\
\hline
Nuclear norm min.    &$\omega(n)$ &$\omega(m\sqrt n)$\\
 (non-private) \cite{shalev2011large}   &  & \\
\hline
Noisy SVD + kNN \cite{MM09}   & -- & --\\
\hline
Noisy SGLD \cite{liu2015fast}   & -- & --\\
\hline
Private FW  (This work) & $\omega(n^{5/4})$ & $\omega(m\sqrt n)$\\
\hline
\end{tabular}
\end{center}
\caption{Sample complexity bounds for matrix completion. $m = $ no. of users, $n =$ no. of items. The bounds hide privacy parameters $\epsilon$ and $\log(1/\delta)$, and polylog factors in $m$, $n$.} \label{tab:bounds}
}
\hfill
\parbox{.53\linewidth}{
\begin{center} 
\begin{tabular}{|c|c|c|}
\hline
{\bf Algorithm} & {\bf Error} \\
\hline 
Randomized response \cite{BDMN05,ChanSS11,dwork2014analyze} &  $O(\sqrt{m+n})$ \\
\hline
Gaussian measurement \cite{HR12}  &  $O\left(\sqrt{m} + \sqrt{\frac{\mu n}{m}}\right)$ \\
\hline
Noisy power method \cite{HR13}   &   $O(\sqrt{\mu})$ \\
\hline
Exponential mechanism \cite{KapralovT13} & $O(m+n)$ \\
\hline
Private FW (This work)& $O\left({m^{3/10}n^{1/10}}\right)$ \\
\hline
Private SVD (This work)&  ${O\left(\sqrt{\mu\left(\frac{n^2}{m}+  \frac{m}{n}\right)}\right)}$ \\
 \hline
\end{tabular}
\end{center}
\caption{Error bounds ($\|Y-Y^*\|_F$) for low-rank approximation. 
$\mu\in[0,m]$ is the \ifnum\supp=1
incoherence parameter (Definition~\ref{def:incoherence}).
\else
\emph{incoherence} parameter.
\fi The bounds hide privacy parameters $\epsilon$ and $\log(1/\delta)$, and polylog factors in $m$ an $n$. 
Rank of the output matrix $Y_{\sf priv}$ is $O\left(m^{2/5}/n^{1/5}\right)$ for Private FW, whereas it is $O(1)$ for the others.}  \label{tab:bounds1}
}
\end{table}

In Table \ref{tab:bounds}, we provide a comparison based on the sample complexity, i.e., the number of users $m$ and the number observed samples $|\Omega|$ needed to attain a generalization error of $o(1)$. We compare our results with the best non-private algorithm for matrix completion based on nuclear norm minimization \cite{shalev2011large}, and the prior work on DP matrix completion \cite{MM09,liu2015fast}. We see that for the same $|\Omega|$, the sample complexity on $m$ increases from $\omega(n)$ to $\omega(n^{5/4})$ for our FW-based algorithm. 
While \cite{MM09,liu2015fast} work under the notion of Joint DP as well, they do not provide any formal accuracy guarantees.

\emph{Interlude: Low-rank approximation.} We also compare our results with the prior work on a related problem of DP low-rank approximation. Given a matrix $Y^*\in\re^{m\times n}$, the goal is to compute a DP low-rank approximation $Y_{\sf priv}$, s.t. $Y_{\sf priv}$ is close to $Y^*$ either in the spectral or Frobenius norm. Notice that this is similar to matrix completion if the set of revealed entries $\Omega$ is the complete matrix. Hence, our methods can be applied directly. To be consistent with the existing literature, we assume that $Y^*$ is rank-one matrix, and each row of $Y^*$ has $\ell_2$-norm at most one 
. Table \ref{tab:bounds1} compares the various results. While all the prior works provide trivial error bounds (in both Frobenius and spectral norm, as $\ltwo{Y^*}=\lfrob{Y^*}\leq \sqrt m$), our methods provide non-trivial bounds. The key difference is that we ensure Joint DP (Definition~\ref{defn:jc}), while existing methods ensure the stricter standard DP (Definition~\ref{defn:dp}), with the exponential mechanism \cite{KapralovT13} ensuring $(\eps,0)$-standard DP.

\section{Background: Notions of privacy}
\label{sec:privDef}

Let $D=\{d_1,\cdots,d_m\}$ be a dataset of $m$ entries. Each entry $d_i$ lies in a fixed domain $\mathcal{T}$, and belongs to an individual $i$, whom we refer to as an \emph{agent} in this paper. Furthermore, $d_i$ encodes potentially sensitive information about agent $i$. Let $\mathcal{A}$ be an algorithm that operates on dataset $D$, and produces a vector of $m$ outputs, one for each agent $i$ and from a set of possible outputs $\mathcal{S}$. Formally, let $\mathcal{A}:\mathcal{T}^m\to\mathcal{S}^m$. Let $D_{-i}$ denote the dataset $D$ without the entry of the $i$-th agent, and similarly $\mathcal{A}_{-i}(D)$ be the set of outputs without the output for the $i$-th agent. Also, let  $(d_i;D_{-i})$ denote the dataset obtained by adding data entry $d_i$ to the dataset $D_{-i}$. In the following, we define both \emph{standard differential privacy} and \emph{\jc}, and contrast them.

\begin{definition}[Standard differential privacy \cite{DKMMN06,DMNS}]\label{defn:dp}
	An algorithm $\mathcal{A}$ satisfies $(\epsilon,\delta)$-differential privacy if for any agent $i$, any  two possible values of data entry $d_i,d'_i\in\mathcal{T}$ for agent $i$, any tuple of data entries for all other agents, $D_{-i}\in\mathcal{T}^{m-1}$, and any output  $S\in\mathcal{S}^m$, we have  
	$$\Pr\limits_{\mathcal{A}}\left[\mathcal{A}\left(d_i;D_{-i}\right)\in S\right]\leq e^\epsilon\Pr\limits_{\mathcal{A}}\left[\mathcal{A}\left(d'_i;D_{-i}\right)\in S\right]+\delta.$$\label{def:diffPrivacy} 
\end{definition}
At a high-level, an algorithm $\A$ is $(\epsilon,\delta)$-standard DP if for any agent $i$ and dataset $D$, the output $\mathcal{A}(D)$ and $D_{-i}$ do not reveal ``much'' about her data entry $d_i$. 
For reasons mentioned in Section \ref{sec:intro}, our matrix completion algorithms provide privacy guarantee based on a relaxed notion of DP, called \emph{\jc}, which was initially proposed in \cite{kearns2014mechanism}. At a high-level, an algorithm $\A$ preserves $(\epsilon,\delta)$-joint DP if for any agent $i$ and dataset $D$,  the output of  $\A$ for the other $(m-1)$ agents (denoted by $\mathcal{A}_{-i}(D)$) and $D_{-i}$ do not reveal ``much'' about her data entry $d_i$. Such a relaxation is necessary for matrix completion because an accurate completion of the row of an agent can reveal a lot of information about her data entry. However, it is still a very strong privacy guarantee for an agent even if every other agent colludes against her, as long as she does not make the predictions made to her public.

\begin{definition}[\Jc \cite{kearns2014mechanism}]\label{defn:jc} \label{def:jDiffPrivacy}
	An algorithm $\mathcal{A}$ satisfies $(\epsilon,\delta)$-joint differential privacy if for any agent $i$, any  two possible values of data entry $d_i,d'_i\in\mathcal{T}$ for agent $i$, any tuple of data entries for all other agents, $D_{-i}\in\mathcal{T}^{m-1}$, and any output  $S\in\mathcal{S}^{m-1}$,   
	$$\Pr\limits_{\mathcal{A}}\left[\mathcal{A}_{-i}\left(d_i;D_{-i}\right)\in S\right]\leq e^\epsilon\Pr\limits_{\mathcal{A}}\left[\mathcal{A}_{-i}\left(d'_i;D_{-i}\right)\in S\right]+\delta.$$
\end{definition}
In this paper, we consider the privacy parameter $\epsilon$ to be a small constant ($\approx 0.1$), and $\delta < 1/m$. There are semantic reasons for such choice of parameters \cite{KS08}, but that is beyond the scope of this work.

\section{Private matrix completion via Frank-Wolfe}
\label{sec:privFW}

Recall that the objective is to solve the matrix completion problem (defined in Section \ref{sec:matrixCompletion}) under Joint DP. A standard modeling assumption is that $Y^*$ is nearly low-rank, leading to the following empirical risk minimization problem \cite{keshavan2010matrix,jain2013low,jin2016provable}: $ \min\limits_{{\sf rank}(Y)\leq k} \underbrace{\frac{1}{2|\Omega|}\|\pomega(Y-Y^*)\|_F^2}_{\widehat{F}(Y)}$, where $k\ll \min(m,n)$. As this is a challenging non-convex optimization problem, a popular approach is to relax the rank constraint to a nuclear-norm constraint, i.e., $\min\limits_{\nuc{Y}\leq k} \widehat{F}(Y)$. 

To this end, we use the FW algorithm 
\ifnum \supp=0
(see the supplementary material
\else
(see Appendix \ref{sec:FW}
\fi
 for more details) as our building block. FW is a popular conditional gradient algorithm in which the current iterate is updated as: $\matT{Y}{t}\leftarrow (1-\eta)\matT{Y}{t-1}+\eta\cdot G$, where $\eta$ is the step size, and $G$ is given by: $\argmin\limits_{ \nuc{G}\leq k}\ \ip{G}{\nabla_{\matT{Y}{t-1}} \widehat{F}(Y)}$.
Note that the optimal solution to the above problem is  $G=-k\mathbf{u}\mathbf{v}^\top$, where ($\lambda$, $\mathbf{u}$, $\mathbf{v}$) are the top singular components of $\matT{A}{t-1}=\pomega(\matT{Y}{t-1}-Y^*)$. Also, the optimal $G$ is a rank-one matrix.

\mypar{Algorithmic ideas} In order ensure Joint DP and still have strong error guarantees, we develop the following ideas. These ideas have been formally compiled into Algorithm \ref{Algo:PrivFW}. Notice that both the functions $\mathcal{A}_{\sf global}$ and $ \mathcal{A}_{\sf local}$ in Algorithm \ref{Algo:PrivFW} are parts of the Private FW technique, where $ \mathcal{A}_{\sf global}$ consists of the global component, and each user runs $\mathcal{A}_{\sf local}$ at her end to carry out a local update. Throughout this discussion, we assume that $\max\limits_{i \in [m]} \ltwo{\pomega(Y^*_i)}\leq L$.

\emph{Splitting the update into global and local components}: One can equivalently write the Frank-Wolfe update as follows: $\matT{Y}{t}\leftarrow (1-\eta)\matT{Y}{t-1}-\eta \cdot \frac{k}{\lambda} \matT{A}{t-1}\mathbf{v} \mathbf{v}^\top$, where $\matT{A}{t-1}, \mathbf{v},$ and $\lambda$ are defined as above. Note that $\mathbf{v}$ and $\lambda^2$ can also be obtained as the top right eigenvector and eigenvalue of $\matT{A}{t-1}^\top \matT{A}{t-1}=\sum\limits_{i=1}^m \matT{A_i}{t-1}^\top \matT{A_i}{t-1}$,  where $\matT{A_i}{t-1}=\pomega(\matT{Y_i}{t-1}-Y_i^*)$ is the $i$-th row of $\matT{A}{t-1}$. We will use the \emph{global component} $ \mathcal{A}_{\sf global}$  in Algorithm \ref{Algo:PrivFW} to compute $\mathbf{v}$ and  $\lambda$. Using the output of $ \mathcal{A}_{\sf global}$, each user (row) $i\in[m]$ can compute her \emph{local update} (using  $ \mathcal{A}_{\sf local}$) as follows:
	{\begin{equation}\matT{Y_i}{t}=(1-\eta)\matT{Y_i}{t-1}- \frac{\eta k}{{\lambda}} \pomega(\matT{Y}{t-1}-Y^*)_i\mathbf{v} \mathbf{v}^\top. \label{eq:abcdas}\end{equation}}
A block schematic of this idea is presented in Figure \ref{fig:JointModel}.

\begin{figure}[ht]
	\begin{center}
		\centerline{\includegraphics[width=0.8\columnwidth]{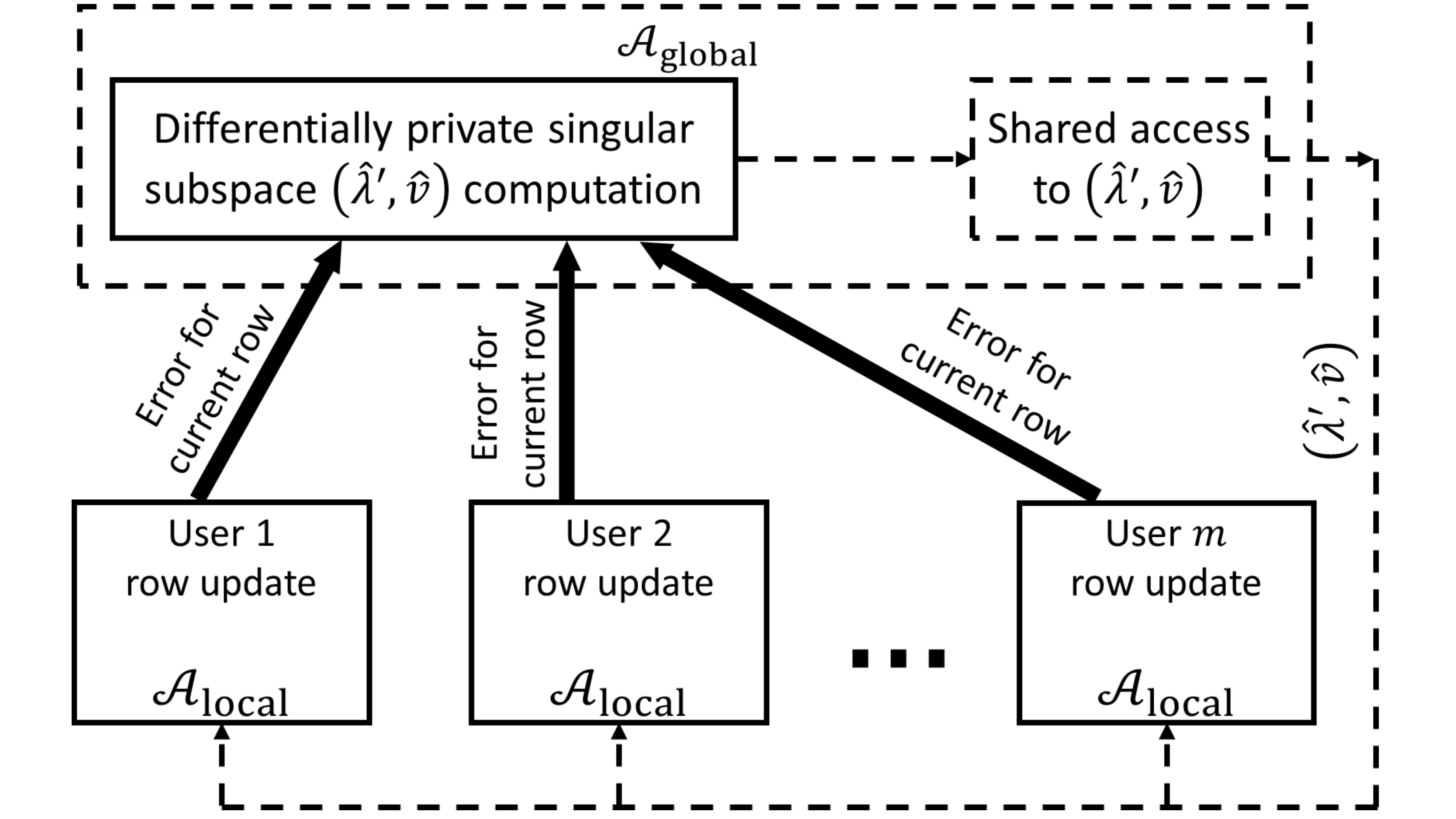}}
	\end{center}
	\caption{Block schematic describing the two functions $\mathcal{A}_{\sf local}$ and $\mathcal{A}_{\sf global}$ of Algorithm \ref{Algo:PrivFW}. The solid boxes and arrows represent computations that are privileged and without external access, and the dotted boxes and arrows represent the unprivileged computations.}
	\label{fig:JointModel}
\end{figure}

\begin{algorithm}
\caption{Private Frank-Wolfe algorithm}
\label{Algo:PrivFW}
\begin{algorithmic}
\FUNCTION{Global Component $\mathcal{A}_{\sf global}$ {\bf (}{\bf Input-} privacy parameters: $(\epsilon,\delta)$ s.t. $\eps \leq 2\log{(1/\delta)}$, total number of iterations: $T$,  bound on $\ltwo{\pomega(Y^*_i)}$: $ L$, failure probability: $\beta$, number of users: $m$, number of items: $n${\bf )}}
\STATE $\sigma\leftarrow L^2 \sqrt{64 \cdot T\log(1/\delta)}/\epsilon, \mathbf{{\widehat v}}\leftarrow \{0\}^{n}$, ${\widehat \lambda}\leftarrow 0$
\FOR{$t\in [T]$}
\STATE $\matT{W}{t}\leftarrow \{0\}^{n\times n}$, ${\widehat \lambda'}\leftarrow  {\widehat \lambda} + \sqrt{\sigma\log(n/\beta)}n^{1/4}$ 
\STATE \textbf{for} $i\in[m]$ \textbf{do} {\small \mbox{$\matT{W}{t}\leftarrow \matT{W}{t}+ \mathcal{A}_{\sf local}(i,\mathbf{\widehat{v}},{\widehat \lambda'},T,t, L)$}}
\STATE $\matT{\widehat{W}}{t}\leftarrow \matT{W}{t}+ \matT{N}{t}$,  where {$\matT{N}{t}\in\re^{n\times n}$} is a matrix with i.i.d. entries from {$\mathcal{N}(0,\sigma^2)$}
\STATE $(\mathbf{\widehat{v}}, {\widehat \lambda}^2)\leftarrow$ Top eigenvector and eigenvalue of $\matT{\widehat{W}}{t}$ \label{step8fw}
\ENDFOR
\ENDFUNCTION


\FUNCTION{Local Update $\mathcal{A}_{\sf local}$ {\bf (Input-} user number: $i$, top right singular vector: $\mathbf{\widehat{v}}$, top singular value: ${\widehat\lambda}'$, total number of iterations: $T$, current iteration: $t$,  bound on $\ltwo{\pomega(Y^*_i)}$:  $L$, private true matrix row: $\pomega(Y^*_i)${\bf )}}
\STATE $\matT{Y_i}{0}\leftarrow \{0\}^n$, $\matT{A_i}{t-1}\leftarrow \pomega(\matT{Y_i}{t-1}-Y_i^*)$
\STATE ${\widehat{u}}_i\leftarrow (\matT{A_i}{t-1}\cdot{\mathbf{\widehat{v}}})/{\widehat \lambda}'$
\STATE Define  $\Pi_{L,\Omega}\left(M\right)_{i,j}=\min\left\{\frac{L}{\ltwo{\pomega\left(M_i\right)}},1\right\} \cdot M_{i,j}$ \label{step:a}
\STATE $\matT{Y_i}{t}\leftarrow \Pi_{L,\Omega}\left(\left(1-\frac{1}{T}\right)\matT{Y_i}{t-1}-\frac{k}{T}{\widehat{u}}_i(\mathbf{\widehat{v}})^T\right)$
\STATE $\matT{A_i}{t}\leftarrow\pomega\left(\matT{Y_i}{t}-Y^*_i\right)$
\STATE \textbf{if} $t=T$, Output $\matT{Y_i}{T}$ as prediction to user $i$ and {\bf stop} 
\STATE \textbf{else} Return $\matT{A_i}{t}^\top \matT{A_i}{t}$ to $\mathcal{A}_{\sf global}$
\ENDFUNCTION
\end{algorithmic}
\end{algorithm}

\emph{Noisy rank-one update}:  Observe that $\mathbf{v}$ and  $\lambda$, the statistics computed in each iteration of $\mathcal{A}_{\sf global}$, are aggregate statistics that use information from all rows of $Y^*$. This ensures that they are noise tolerant. Hence, adding sufficient noise can ensure standard DP (Definition \ref{def:diffPrivacy}) for $\mathcal{A}_{\sf global}$. \footnote{The second term in computing ${\widehat \lambda'}$ in Algorithm~\ref{Algo:PrivFW} is due to a bound on the spectral norm of the Gaussian noise matrix. We use this bound to control the error introduced in the computation of $\hat\lambda$.}
	Since the final objective is to satisfy Joint DP (Definition \ref{def:jDiffPrivacy}), the local component $\mathcal{A}_{\sf local}$ can compute the update for each user (corresponding to \eqref{eq:abcdas}) without adding any noise.

\emph{Controlling norm via projection}: In order to control the amount of noise needed to ensure DP, any individual data entry (here, any row of $Y^*$) should have a bounded effect on the aggregate statistic computed by $\mathcal{A}_{\sf global}$. However, each intermediate computation $\matT{Y_i}{t}$ in \eqref{eq:abcdas} can have high $\ell_2$-norm even if $\ltwo{\pomega(Y^*_i)}\leq L$. We address this by applying a projection operator $\Pi_{L,\Omega}$  (defined below) to $\matT{Y_i}{t}$, and compute the local update as { $\Pi_{L,\Omega}\left(\matT{Y_i}{t}\right)$} in place of \eqref{eq:abcdas}. $\Pi_{L,\Omega}$ is defined as follows: For any matrix $M$, $\Pi_{L,\Omega}$ ensures that any row of the ``zeroed out'' matrix $\pomega{(M)}$ does not have $\ell_2$-norm higher than $L$. Formally, $\Pi_{L,\Omega}\left(M\right)_{i,j}=\min\left\{\frac{L}{\ltwo{\pomega\left(M_i\right)}},1\right\} \cdot M_{i,j}$ for all entries $(i,j)$ of $M$. In our analysis, we show that this projection operation does not increase the error.

\subsection{Privacy and utility analysis}
\label{sec:privUtil}

\begin{theorem}
	Algorithm~\ref{Algo:PrivFW} satisfies $(\epsilon,\delta)$-joint DP. 
	\label{thm:priv}
\end{theorem}

\ifnum\supp=0
We defer the proof to the supplementary material.
\else
For a proof of Theorem \ref{thm:priv}, see Appendix \ref{app:privFWpriv}.
\fi
The proof uses standard DP properties of Gaussian noise addition from \cite{bun2016concentrated}. The requirement $\eps \leq 2\log{(1/\delta)}$ in the input of Algorithm~\ref{Algo:PrivFW} is due to a reduction of a Concentrated DP guarantee to a standard DP guarantee. We now show that the empirical risk of our algorithm is close to the optimal as long as the number of users $m$ is ``large''.

\begin{theorem}[Excess empirical risk guarantee]
	Let $Y^*$ be a matrix with {$\nuc{Y^*}\leq k$}, and {$\max\limits_{i \in [m]} \ltwo{\pomega(Y^*)_i}\leq L$.} Let $\matT{\Ypriv}{T}$ be a matrix, with its rows  being $\matT{\Ypriv_i}{T}$ for all $i\in[m]$, computed by function $\mathcal{A}_{\sf local}$  in Algorithm \ref{Algo:PrivFW} after $T$ iterations. If $\eps \leq 2\log{\left(\frac{1}{\delta}\right)}$, then with probability at least 2/3 over the outcomes of Algorithm \ref{Algo:PrivFW}, the following is true:
	{\small\begin{align*}
	 \widehat{F}\left(\matT{\Ypriv}{T}\right)
	 =O\left(\frac{k^2}{|\Omega|T}+\frac{kT^{1/4}L\sqrt{n^{1/2}\log^{1/2}(1/\delta)\log n}}{|\Omega|\sqrt{\epsilon}}\right). 
	\end{align*}}
	Furthermore, if  $T=\tilde O\left (\frac{k^{4/5}\epsilon^{2/5}}{n^{1/5}L^{4/5}}\right)$, then $\widehat{F}\left(\matT{\Ypriv}{T}\right)  =\tilde O\left(\frac{k^{6/5}n^{1/5}L^{4/5}}{|\Omega|\epsilon^{2/5}}\right)$ after hiding poly-logarithmic terms.
	\label{thm:utilRransferLearning}
\end{theorem}

\ifnum\supp=0
We defer the proof to the supplementary material.
\else
See Appendix~\ref{app:privFW} for a proof of Theorem \ref{thm:utilRransferLearning}.
\fi
At a high-level, our proof combines the noisy eigenvector estimation error for Algorithm~\ref{Algo:PrivFW} with a noisy-gradient analysis of the FW algorithm. Also, note that the first term in the bound corresponds to the standard FW convergence error, while the second term can be attributed to the noise added for DP which directly depends on $T$. We also compute the optimal number of iterations required to minimize the empirical risk. Finally,  the rank of $\matT{\Ypriv}{T}$ is at most $T$, but its nuclear-norm is bounded by $k$. As a result, $\matT{\Ypriv}{T}$ has low {\em generalization error} (see Section \ref{sec:genError}).

\begin{remark}
	\label{rem:fw1}
	We further illustrate our empirical risk bound by considering a simple setting: let $Y^*$ be a rank-one matrix with $Y_{ij}^*\in [-1,1]$ and $|\Omega|=m\sqrt{n}$. Then $k=O(\sqrt{mn})$, and $L=O(n^{1/4})$, implying an error of $\widetilde{O}\left(\sqrt{n}m^{-2/5}\right)$ hiding the privacy parameter $\eps$; in contrast, a trivial solution like $Y=0$ leads to $O(1)$ error. Naturally, the error increases with $n$ as there is more information to be protected. However, it decreases with a larger number of users $m$ as the presence/absence of a user has lesser effect on the solution with increasing $m$. We leave further investigation into the dependency of the error on $m$ for future work.
\end{remark}

\begin{remark}
	Our analysis does not require an upper bound on the nuclear norm of $Y^*$ (as stated in Theorem~\ref{thm:utilRransferLearning}); we would instead incur an additional error of $\min\limits_{\nuc{Y}\leq k}\frac{1}{|\Omega|}\lfrob{\pomega\left(Y^*-Y\right)}^2$. Moreover, consider a similar scenario as in Remark~\ref{rem:fw1}, but $|\Omega|=mn$, i.e., all the entries of $Y^*$ are revealed. In such a case,  $L=O(\sqrt{n})$, and the problem reduces to that of standard  low-rank matrix approximation of $Y^*$. Note that our result here leads to an error bound of $\widetilde{O}\left(n^{1/5}m^{-2/5}\right)$, while the state-of-the-art result by \cite{HR13} leads to an error bound of $O(1)$ due to being in the much stricter standard DP model. 
	\label{rem:12}
\end{remark}

\subsubsection{Generalization error guarantee}
\label{sec:genError} 
We now present a generalization error (defined in Equation~\ref{eqn: gen_error_def}) bound which shows that our approach provides accurate prediction  over {\em unknown} entries. 
For obtaining our bound, we use Theorem 1 from \cite{srebro2005rank} (provided in 
\ifnum\supp=0
the supplementary material
\else
Appendix~\ref{app:res}
\fi 
for reference).
Also, the output of Private FW (Algorithm \ref{Algo:PrivFW}) has rank at most $T$, where $T$ is the number of iterations. Thus, replacing $T$ from Theorem \ref{thm:utilRransferLearning}, we get the following:
\begin{corollary}[Generalization Error]
 Let $\nuc{Y^*}\leq k$ for a hidden matrix $Y^*$, and $\ltwo{\pomega(Y_i^*)} \leq L$ for every row $i$ of $Y^*$. If we choose the number of rounds in Algorithm \ref{Algo:PrivFW} to be $O\left(\frac{k^{4/3}}{\left(|\Omega| (m+n)\right)^{1/3}}\right)$, the data samples in $\Omega$ are drawn u.a.r. from $[m]\times[n]$, and $\eps \leq 2\log{\left(\frac{1}{\delta}\right)}$, then with probability at least 2/3 over the outcomes of the algorithm and choosing $\Omega$, the following is true for the final completed matrix $Y$:
	{\small
	$$F(Y)=\tilde O\left(\frac{k^{4/3} L n^{1/4}}{\sqrt{\epsilon|\Omega|^{13/6}(m+n)^{1/6}}}+\left(\frac{k\sqrt{m+n}}{|\Omega|}\right)^{2/3}\right).	$$}The $\tilde O\left(\cdot\right)$ hides poly-logarithmic terms in $m,n,|\Omega|$ and $\delta$. 
	\label{cor:abcfd12}
\end{corollary}

\begin{remark}\label{rem:ge1}
	We further illustrate our bound using a setting similar to the one considered in Remark~\ref{rem:fw1}. Let $Y^*$ be a rank-one matrix with $Y^*_{ij}\in [-1,1]$ for all $i,j$; let $|\Omega|\geq m\sqrt{n}\cdot{\sf polylog}(n)$, i.e., the fraction of movies rated by each user is arbitrarily small for larger $n$. For this setting, our generalization error is $o(1)$ for $m=\omega(n^{5/4})$. This is slightly higher than the bound in the non-private setting by  \cite{shalev2011large}, where $m=\omega(n)$ is sufficient to get generalization error $o(1)$. Also, as the first term in the error bound pertains to DP, it decreases with a larger number of users $m$, and increases with $n$ as it has to preserve privacy of a larger number of items. In contrast, the second term is the matrix completion error decreases with $n$. This is intuitive, as a larger number of movies enables more sharing of information between users, thus allowing a better estimation of preferences $Y^*$. However, just increasing $m$ may not always lead to a more accurate solution (for example, consider the case of $n=1$). 
\end{remark}
\begin{remark}
	The guarantee in Corollary \ref{cor:abcfd12} is for uniformly random $\Omega$, but using the results of \cite{shamir2011collaborative}, it is straightforward to extend our results to any i.i.d. distribution over $\Omega$. Moreover, we can extend our results to handle strongly convex and smooth loss functions instead of the squared loss considered in this paper. 
\end{remark}

\subsection{Efficient PCA via Oja's Algorithm}
Algorithm~\ref{Algo:PrivFW} requires computing the top eigenvector of $\matT{\widehat{W}}{t}={\matT{{W}}{t}}+\matT{N}{t}$, where $\matT{{W}}{t}=\sum_i \left(\matT{A_i}{t}\right)^\top \matT{A_i}{t}$ and $\matT{N}{t}$ is a random noise matrix. However, this can be a bottleneck for computation as $\matT{N}{t}$ itself is a dense $n\times n$ matrix, implying a  space complexity of $\Omega(n^2+m k)$, where $k$ is the maximum number of ratings provided by a user. Similarly, standard eigenvector computation algorithms will require $O(m k^2+n^2)$ time (ignoring factors relating to rate of convergence), which can be prohibitive for practical recommendation systems with large $n$.  We would like to stress that this issue plagues even standard DP PCA algorithms \cite{DTTZ13}, which have quadratic space-time complexity in the number of dimensions. 

We tackle this by using a stochastic algorithm for the top eigenvector computation that significantly reduces both space and time complexity while preserving privacy. In particular, we use Oja's algorithm \cite{JJKNS16}, which computes top eigenvectors of a matrix with a stochastic access to the matrix itself. That is, if we want to compute the top eigenvector of $\matT{{W}}{t}$, we can use the following updates: 
\begin{equation}\label{eq:oja}\mathbf{\widehat{v}_\tau}=(I+\eta X_\tau)\mathbf{\widehat{v}_{\tau-1}}, \qquad \mathbf{\widehat{v}_\tau}=\mathbf{\widehat{v}_\tau}/\|\mathbf{\widehat{v}_\tau}\|_2\end{equation}
where $\E[X_\tau]=\matT{{W}}{t}$. For example, we can update $\mathbf{\widehat{v}_\tau}$ using $X_\tau=\matT{{W}}{t}+{N^{(t)}_\tau}$ where each entry of ${N^{(t)}_\tau}$ is sampled i.i.d. from a Gaussian distribution calibrated to ensure DP. 
Even this algorithm in its current form does not decrease the space or time complexity as we need to generate a dense matrix $\matT{N_\tau}{t}$ in each iteration. However, by observing that $\matT{N_\tau}{t} v=g_\tau\sim \mathcal{N}(0, \sigma^2 \mathbf{1}^n)$ where $v$ is independent of $\matT{N_\tau}{t}$, we can now replace the generation of $\matT{N_\tau}{t}$ by the generation of a vector $g_\tau$, thus reducing both the space and time complexity of our algorithm. The computation of each update is significantly {\em cheaper} as long as $m k\ll n^2$, which is the case for practical recommendation systems as $k$ tends to be fairly small there (typically on the order of $\sqrt{n}$). 

\begin{algorithm}[tb]
	\caption{Private Oja's algorithm}
	\label{Alg:Eigs}
	\begin{algorithmic}
		\STATE {\bfseries Input:} An $m \times n$ matrix $A$ s.t. each row $\ltwo{A_i}\leq L$, privacy parameters: $(\epsilon,\delta)$ s.t. $\epsilon\leq 2\log(1/\delta)$, total number of iterations: $\Gamma$ 
		\STATE $\sigma \leftarrow L^2\sqrt{256\cdot \Gamma\log(2/\delta)}/\epsilon, \mathbf{\widehat{v}_0} \sim \mathcal{N}(0,\sigma^2 I)$
		\FOR{$\tau\in[\Gamma]$}
			\STATE $\eta=\frac{1}{\Gamma \sigma \sqrt{n}}, g_\tau\sim \mathcal{N}(0,\sigma^2\mathbf{1}^n)$
			\STATE $\mathbf{\widehat{v}_\tau}\leftarrow \mathbf{\widehat{v}_{\tau-1}} + \eta\left(A^T A\mathbf{\widehat{v}_{\tau-1}}+g_\tau\right)$,   $\mathbf{\widehat{v}_\tau}\leftarrow \mathbf{\widehat{v}_\tau}/\|\mathbf{\widehat{v}_\tau}\|_2$ 
		\ENDFOR
		\STATE Return $\mathbf{\widehat{v}_\Gamma}$, $\left(\widehat{\lambda}_\Gamma^2 \leftarrow ||A \cdot \mathbf{\widehat{v}_\Gamma}||_2^2 + \mathcal{N}(0,\sigma^2)\right)$
	\end{algorithmic}
\end{algorithm}


Algorithm~\ref{Alg:Eigs} provides a pseudocode of the eigenvector computation method. The computation of the approximate eigenvector $\mathbf{\widehat{v}_\Gamma}$ and the eigenvalue $\widehat{\lambda}_\Gamma^2$ in it is DP (directly follows via the proof of Theorem \ref{thm:priv}). The next natural question is how well can $\mathbf{\widehat{v}_\Gamma}$ approximate the behavior of the top eigenvector of the non-private covariance matrix $\matT{{W}}{t}$? To this end, we provide Theorem~\ref{thm:oja} below 
that analyzes Oja's algorithm, and shows that the \emph{Rayleigh quotient} of the approximate eigenvector is close to the top eigenvalue of $\matT{{W}}{t}$. 
In particular, using Theorem~\ref{thm:oja} along with the fact that in our case, $\mathcal{V}=\sigma^2 n$, we have 
$\ltwo{\matT{A}{t}}^2\leq \|\matT{A}{t}\mathbf{\widehat{v}_\Gamma}\|_2^2+ O\left(\sigma \sqrt{n}\log(\eta/\beta)\right)$ with high probability (w.p. $\geq 1-\beta^2)$),
where $\mathbf{\widehat{v}_\Gamma}$ is the output of Algorithm~\ref{Alg:Eigs}, $\Gamma=\Omega\left(\min\left\{\frac{1}{\beta},\frac{\|\matT{A}{t}\|^2}{\sigma\sqrt{n}}\right\}\right)$, and $\eta=\frac{1}{\Gamma\cdot \sigma \sqrt{n}}$.  

Note that the above given bound is exactly the bound required in the proof of \ifnum\supp=0 Theorem \ref{thm:utilRransferLearning}. \else  Theorem~\ref{thm:utilFW} in Appendix~\ref{sec:FW}. \fi
  Hence, computing the top eigenvector privately using Algorithm~\ref{Alg:Eigs} does not change the utility bound of Theorem~\ref{thm:utilRransferLearning}. 

\begin{theorem}[Based on Theorem 3 \cite{zeyuan}]\label{thm:oja}
	Let $X_1$, $X_2, \dots X_\Gamma$ be sampled i.i.d. such that $\E{X_i}=W=A^T A$. Moreover, let {\small$\mathcal{V}=\max\{\|\E{(X_i-W)^T(X_i-W)}\|, \|\E{(X_i-W)(X_i-W)^T}\|\}$}, and $\eta=\frac{1}{\sqrt{\mathcal{V}\Gamma}}$. Then, the $\Gamma$-th iterate of Oja's Algorithm (Update \eqref{eq:oja}) , i.e., $\mathbf{\widehat{v}_\Gamma}$, satisfies (w.p. $\geq 1-1/{\sf poly}(\Gamma)$): $\mathbf{\widehat{v}_\Gamma}^T W \mathbf{\widehat{v}_\Gamma}\geq  \|W\|_2- O\left(\sqrt{\frac{\mathcal{V}}{\Gamma}}+\frac{\|W\|_2}{\Gamma}\right).$
\end{theorem}

\emph{Comparison with Private Power Iteration (PPI) method \cite{HR13}:} Private PCA via PPI provides utility guarantees dependent on the gap between the top and the $k$th eigenvalue of the input matrix $A$ for some $k > 1$, whereas private Oja's utility guarantee is gap-independent.
\section{Experimental evaluation}
\label{sec:exptVal}
We now present empirical results for Private FW (Algorithm~\ref{Algo:PrivFW}) on several benchmark datasets, 
and compare its performance to state-of-the-art methods like  \cite{MM09}, and private as well as non-private variant of the Projected Gradient Descent (PGD) method \cite{cai2010singular, BST, DPDL}. In all our experiments, we see that private FW provides accuracy very close to that of the non-private baseline, and almost always significantly outperforms both the private baselines. 


\emph{Datasets}: As we want to preserve privacy of every user, and the output for each user is $n$-dimensional, we can expect the private recommendations to be accurate only when $m\gg n$ (see Theorem~\ref{thm:priv}). Due to this constraint, we conduct experiments on the following datasets: 
1) \emph{Synthetic:} We generate a random rank-one matrix $Y^*=uv^T$ with  unit $\ell_\infty$-norm, $m=500$K, and $n=400$, 2) \emph{Jester:} This dataset contains $n=100$ jokes, and $m \approx 73$K users, 3) \emph{MovieLens10M (Top 400):}  We pick the $n=400$ most rated movies from the Movielens10M dataset, resulting in $m\approx 70$K users, 
4) \emph{Netflix (Top 400):} We pick the $n=400$ most rated movies from the Netflix prize dataset, resulting in $m\approx474$K users, 
and 5) \emph{Yahoo! Music (Top 400):}  We pick the $n=400$ most rated songs from the Yahoo! music dataset, resulting in $m \approx 995$K users.
\footnote{\label{rem:abbd1}For $n=900$ with all the considered datasets (except Jester), we see that private PGD takes too long to complete; we present an evaluation for the other algorithms in \ifnum\supp=0 the supplementary material. \else Appendix \ref{app:exptVal}. \fi} 
We rescale the ratings to be from 0 to 5 for Jester and Yahoo! Music.

\emph{Procedure: } For all datasets, we randomly sample $1\%$ of the given ratings for measuring the test error. 
 For experiments with privacy, for all datasets except Jester, we randomly select at most $\xi = 80$ ratings per user to get $\pomega(Y^*)$. We vary the privacy parameter $\eps\in [0.1,5]$ \footnote{The requirement in Algorithm \ref{Algo:PrivFW} that $\eps \leq 2\log{(1/\delta)}$ is satisfied by all the values of $\eps$ considered for the experiments.}, but keep $\delta=10^{-6}$, thus ensuring that $\delta < \frac{1}{m}$ for all datasets. Moreover, we report results averaged over $10$ independent runs. 

Note that the privacy guarantee is user-level, which 
 effectively translates to an entry-level guarantee of $\eps_{entry} = \frac{\eps_{user}}{\xi}$, i.e., $\eps_{entry} \in [ 0.00125, 0.0625]$ as $\eps_{user} \in [0.1,5]$. 


For the experiments with private Frank-Wolfe (Algorithm~\ref{Algo:PrivFW}), we normalize the data as $\hat{r}_{i,j} = r_{i,j} - u_i$ for all $i \in [m], j \in [n]$, where $r_{i,j}$ is user $i$'s rating for item $j$, and $u_i$ is the average rating of user $i$. Note that each user can safely perform such a normalization at her end without incurring any privacy cost. Regarding the parameter choices for private FW, we cross-validate over the nuclear norm bound $k$, and the number of iterations $T$ for each dataset. For $k$, we set it to the actual nuclear norm for the synthetic dataset, and choose from  $\{20000, 25000\}$ for Jester, $\{ 120000, 130000\}$ for Netflix,  $\{30000, 40000\}$ for MovieLens10M, and $\{130000, 150000\}$ for the Yahoo! Music dataset.  We choose $T$ from various values in $[5,50]$. Consequently, the rank of the prediction matrix for all the private FW experiments is at most 50. For faster training, we calibrate the scale of the noise in every iteration according to the number of iterations that the algorithm has completed, while still ensuring the overall DP guarantee.

\emph{Non-private baseline:} For the non-private baseline, we normalize the training data for the experiments with non-private Frank-Wolfe by removing the per-user and per-movie averages (as in \cite{jaggi2010simple}), and we run non-private FW for 400 iterations. For non-private PGD, we tune the step size schedule. We find that non-private FW and non-private PGD converge to the same accuracy after tuning, and hence, we use this as our baseline.   

\emph{Private baselines:} To the best of our knowledge, only \cite{MM09} and \cite{liu2015fast} address the user-level DP matrix completion problem. While we present an empirical evaluation of the `SVD after cleansing method' from the former, we refrain from comparing to the latter \footnote{The exact privacy parameters $(\epsilon$ and $\delta)$ for the  Stochastic Gradient Langevin Dynamics based algorithm in \cite{liu2015fast} (correspondigly, in \cite{WFS15}) are unclear. They use a Markov chain based sampling method; to obtain quantifiable $(\eps, \delta)$, the sampled distribution is required to  converge  (non-asymptotically) to a DP preserving distribution in $\ell_1$ distance, for which we are not aware of any analysis.}. We also provide a comparison with private PGD (pseudocode provided in \ifnum\supp=0
 the supplementary material).
\else
 Appendix \ref{app:Pseudo}).
\fi 

For the `SVD after cleansing method' from \cite{MM09}, we set $\delta = 10^{-6}$, and select $\epsilon$ appropriately to ensure a fair comparison. We normalize the data by removing the private versions of the global average rating and the per-movie averages. We tune the shrinking parameters $\beta_m$ and $\beta_p$ from various values in $[5,15]$, and $\beta$ from $[5,25]$. For private PGD,  we tune $T$ from various values in $[5,50]$, and the step size schedule from $\left \{t^{-1/2}, t^{-1}, 0.05, 0.1, 0.2, 0.5\right \}$ for $t \in [T]$. We set the nuclear norm constraint $k$ equal to the nuclear norm of the hidden matrix, and for faster training, we calibrate the scale of the noise as in our private FW experiments.

\begin{figure}[t]
	\centering
	\begin{tabular}{ccc}
		\hspace*{-15pt}
	\begin{minipage}[b]{0.33\textwidth}
		\includegraphics[width=\textwidth]{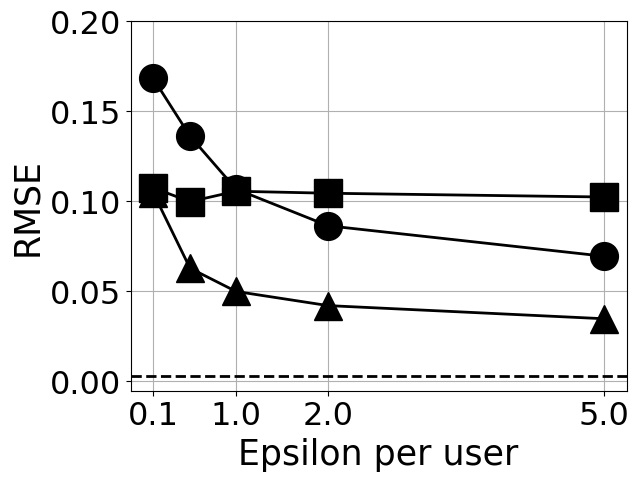}
	\end{minipage}
	&		\hspace*{-15pt}
	\begin{minipage}[b]{0.33\textwidth}
		\includegraphics[width=\textwidth]{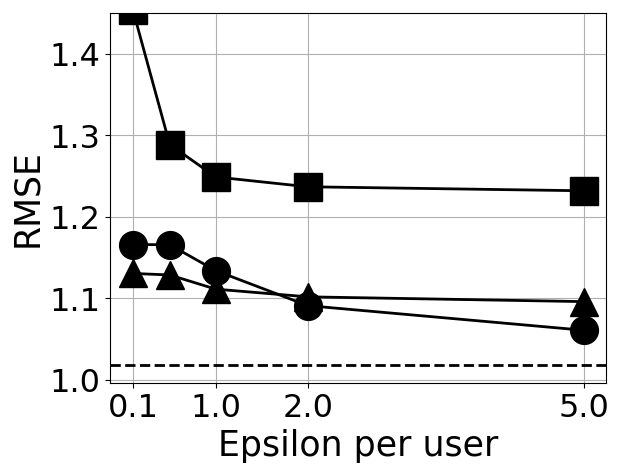}
	\end{minipage}
	& \hspace*{-15pt}
	\begin{minipage}[b]{0.33\textwidth}
		\includegraphics[width=\textwidth]{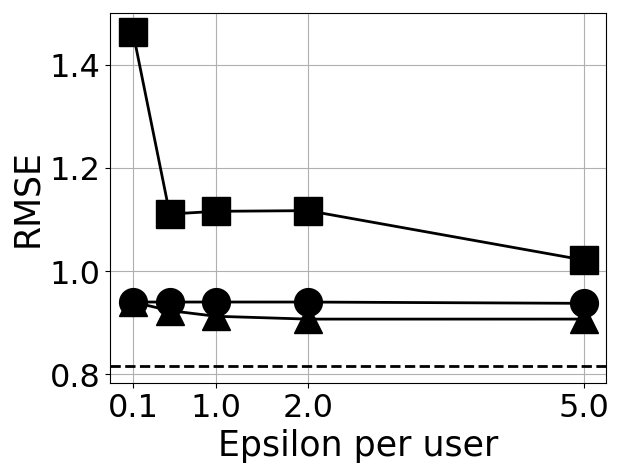}
	\end{minipage}\\
(a)&(b)&(c)\\
\hspace*{-15pt}
	\begin{minipage}[b]{0.33\textwidth}
	\includegraphics[width=\textwidth]{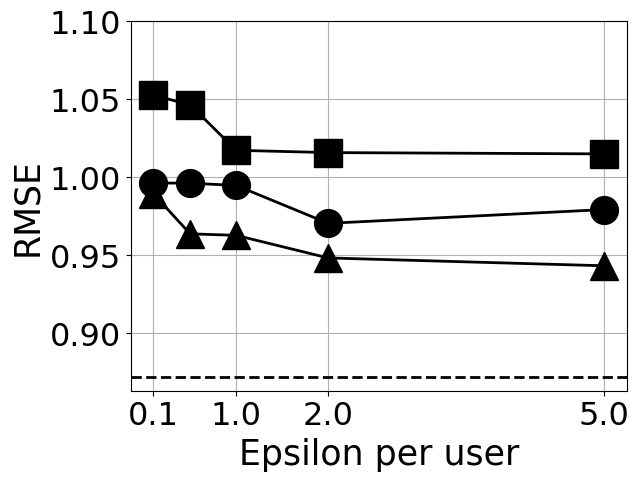}
\end{minipage}
& \hspace*{-15pt}
\begin{minipage}[b]{0.33\textwidth}
	\includegraphics[width=\textwidth]{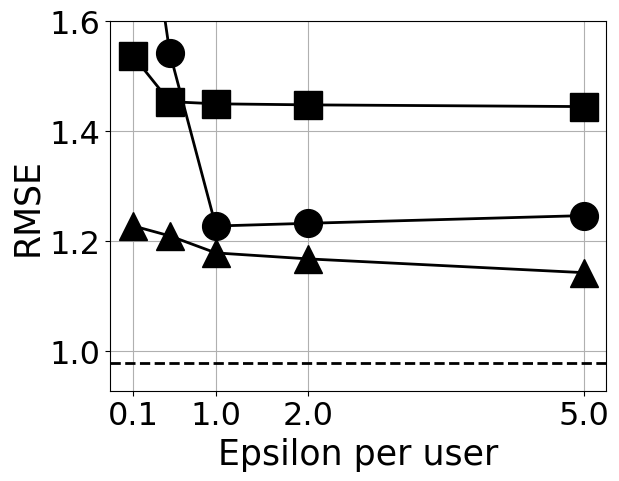}
\end{minipage}
&		\hspace*{-15pt}
\begin{minipage}[b]{0.33\textwidth}
	\includegraphics[width=\textwidth]{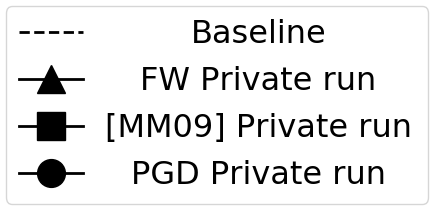}
\end{minipage}\\
(d)&(e)&(f)
\end{tabular}
	\caption{Root mean squared error (RMSE) vs. $\eps$, on (a) synthetic, (b) Jester, (c) MovieLens10M, (d) Netflix, and (e) Yahoo! Music datasets, for $\delta = 10^{-6}$. A legend for all the plots is given in (f).}
	\label{fig:syn}
\end{figure}

\emph{Results:} Figure \ref{fig:syn} shows the results of our experiments\footnote{In all our experiments, the implementation of private FW with Oja's method (Algorithm \ref{Alg:Eigs}) did not suffer any perceivable loss of accuracy as compared to the variant in Algorithm \ref{Algo:PrivFW}; all the plots in Figure~\ref{fig:syn} remain identical.}. 
Even though all the considered private algorithms satisfy Joint DP, our private FW method almost always incurs a significantly lower test RMSE than the two private baselines. Note that although non-private PGD provides similar empirical accuracy as non-private FW, the difference in performance for their private versions can be attributed to the noise being calibrated to a rank-one update for our private Frank-Wolfe.

\section{Future directions}
\label{sec:Discussion}

For future work, it is interesting to understand the optimal dependence of the generalization error for DP matrix completion w.r.t. the number of users and the number of items. Also, extending our techniques to other popular matrix completion methods, like alternating minimization, is another promising direction. 

\section*{Acknowledgements}

The authors would like to thank Ilya Mironov, and the anonymous reviewers, for their helpful comments. This material is in part based upon work supported by NSF grants CCF-1740850 and IIS-1447700, and a grant from the Sloan foundation. 
\bibliographystyle{plain}
\bibliography{reference}

\begin{thebibliography}{10}

\bibitem{DPDL}
Martin Abadi, Andy Chu, Ian Goodfellow, H.~Brendan McMahan, Ilya Mironov, Kunal
  Talwar, and Li~Zhang.
\newblock Deep learning with differential privacy.
\newblock In {\em Proceedings of the 2016 ACM SIGSAC Conference on Computer and
  Communications Security}, CCS '16, pages 308--318, New York, NY, USA, 2016.
  ACM.

\bibitem{zeyuan}
Zeyuan Allen{-}Zhu and Yuanzhi Li.
\newblock First efficient convergence for streaming k-pca: {A} global,
  gap-free, and near-optimal rate.
\newblock In {\em 58th {IEEE} Annual Symposium on Foundations of Computer
  Science, {FOCS} 2017, Berkeley, CA, USA, October 15-17, 2017}, pages
  487--492, 2017.

\bibitem{BST}
Raef Bassily, Adam Smith, and Abhradeep Thakurta.
\newblock Private empirical risk minimization: Efficient algorithms and tight
  error bounds.
\newblock In {\em Foundations of Computer Science (FOCS), 2014 IEEE 55th Annual
  Symposium on}, pages 464--473. IEEE, 2014.

\bibitem{BassilyST14}
Raef Bassily, Adam~D. Smith, and Abhradeep Thakurta.
\newblock Private empirical risk minimization, revisited.
\newblock {\em CoRR}, abs/1405.7085, 2014.

\bibitem{Netf}
James Bennett, Stan Lanning, and Netflix Netflix.
\newblock The netflix prize.
\newblock In {\em In KDD Cup and Workshop in conjunction with KDD}, 2007.

\bibitem{BDMN05}
Avrim Blum, Cynthia Dwork, Frank McSherry, and Kobbi Nissim.
\newblock Practical privacy: the sulq framework.
\newblock In {\em Proceedings of the twenty-fourth ACM SIGMOD-SIGACT-SIGART
  symposium on Principles of database systems}, pages 128--138. ACM, 2005.

\bibitem{bun2016concentrated}
Mark Bun and Thomas Steinke.
\newblock Concentrated differential privacy: Simplifications, extensions, and
  lower bounds.
\newblock In {\em TCC}, 2016.

\bibitem{svt}
Jian{-}Feng Cai, Emmanuel~J. Cand{\`{e}}s, and Zuowei Shen.
\newblock A singular value thresholding algorithm for matrix completion.
\newblock {\em {SIAM} Journal on Optimization}, 20(4):1956--1982, 2010.

\bibitem{cai2010singular}
Jian-Feng Cai, Emmanuel~J Cand{\`e}s, and Zuowei Shen.
\newblock A singular value thresholding algorithm for matrix completion.
\newblock {\em SIAM Journal on Optimization}, 2010.

\bibitem{CCS10}
Jian{-}Feng Cai, Emmanuel~J. Cand{\`{e}}s, and Zuowei Shen.
\newblock A singular value thresholding algorithm for matrix completion.
\newblock {\em {SIAM} Journal on Optimization}, 2010.

\bibitem{CalandrinoKNFS11}
Joseph~A. Calandrino, Ann Kilzer, Arvind Narayanan, Edward~W. Felten, and
  Vitaly Shmatikov.
\newblock ``you might also like'': Privacy risks of collaborative filtering.
\newblock In {\em IEEE Symposium on Security and Privacy}, 2011.

\bibitem{candes2012exact}
Emmanuel Candes and Benjamin Recht.
\newblock Exact matrix completion via convex optimization.
\newblock {\em Communications of the ACM}, 2012.

\bibitem{ChanSS11}
T.-H.~Hubert Chan, Elaine Shi, and Dawn Song.
\newblock Private and continual release of statistics.
\newblock {\em ACM Trans. Inf. Syst. Secur.}, 14(3):26, 2011.

\bibitem{clarkson2010coresets}
Kenneth~L Clarkson.
\newblock Coresets, sparse greedy approximation, and the frank-wolfe algorithm.
\newblock {\em ACM Transactions on Algorithms (TALG)}, 2010.

\bibitem{DN03}
Irit Dinur and Kobbi Nissim.
\newblock Revealing information while preserving privacy.
\newblock In {\em Proceedings of the Twenty-Second {ACM} {SIGACT-SIGMOD-SIGART}
  Symposium on Principles of Database Systems, June 9-12, 2003, San Diego, CA,
  {USA}}, pages 202--210, 2003.

\bibitem{DKMMN06}
Cynthia Dwork, Krishnaram Kenthapadi, Frank McSherry, Ilya Mironov, and Moni
  Naor.
\newblock Our data, ourselves: Privacy via distributed noise generation.
\newblock In {\em EUROCRYPT}, 2006.

\bibitem{DMNS}
Cynthia Dwork, Frank McSherry, Kobbi Nissim, and Adam Smith.
\newblock Calibrating noise to sensitivity in private data analysis.
\newblock In {\em Theory of Cryptography Conference}, pages 265--284. Springer,
  2006.

\bibitem{DR14}
Cynthia Dwork, Aaron Roth, et~al.
\newblock The algorithmic foundations of differential privacy.
\newblock {\em Foundations and Trends in Theoretical Computer Science},
  9(3-4):211--407, 2014.

\bibitem{DTTZ13}
Cynthia Dwork, Kunal Talwar, Abhradeep Thakurta, and Li~Zhang.
\newblock Randomized response strikes back: Private singular subspace
  computation with (nearly) optimal error guarantees.
\newblock 2013.

\bibitem{dwork2014analyze}
Cynthia Dwork, Kunal Talwar, Abhradeep Thakurta, and Li~Zhang.
\newblock Analyze gauss: optimal bounds for privacy-preserving principal
  component analysis.
\newblock In {\em Proceedings of the 46th Annual ACM Symposium on Theory of
  Computing}, pages 11--20. ACM, 2014.

\bibitem{erlingsson2014rappor}
{\'U}lfar Erlingsson, Vasyl Pihur, and Aleksandra Korolova.
\newblock Rappor: Randomized aggregatable privacy-preserving ordinal response.
\newblock In {\em CCS}, 2014.

\bibitem{FW56}
Marguerite Frank and Philip Wolfe.
\newblock An algorithm for quadratic programming.
\newblock {\em Naval research logistics quarterly}, 3(1-2):95--110, 1956.

\bibitem{Jester}
Ken Goldberg, Theresa Roeder, Dhruv Gupta, and Chris Perkins.
\newblock Eigentaste: A constant time collaborative filtering algorithm.
\newblock {\em Inf. Retr.}, 4(2):133--151, July 2001.

\bibitem{HR12}
Moritz Hardt and Aaron Roth.
\newblock Beating randomized response on incoherent matrices.
\newblock In {\em STOC}, 2012.

\bibitem{HR13}
Moritz Hardt and Aaron Roth.
\newblock Beyond worst-case analysis in private singular vector computation.
\newblock In {\em STOC}, 2013.

\bibitem{HW14}
Moritz Hardt and Mary Wootters.
\newblock Fast matrix completion without the condition number.
\newblock In {\em COLT}, 2014.

\bibitem{Harper2}
F.~Maxwell Harper and Joseph~A. Konstan.
\newblock The movielens datasets: History and context.
\newblock {\em ACM Trans. Interact. Intell. Syst.}, 2015.

\bibitem{Jag13}
Martin Jaggi.
\newblock Revisiting frank-wolfe: Projection-free sparse convex optimization.
\newblock In {\em ICML}, pages 427--435, 2013.

\bibitem{jaggi2010simple}
Martin Jaggi and Marek Sulovsky.
\newblock A simple algorithm for nuclear norm regularized problems.
\newblock In {\em ICML}, 2010.

\bibitem{JJKNS16}
Prateek Jain, Chi Jin, Sham~M Kakade, Praneeth Netrapalli, and Aaron Sidford.
\newblock Streaming pca: Matching matrix bernstein and near-optimal finite
  sample guarantees for oja’s algorithm.
\newblock In {\em Conference on Learning Theory}, pages 1147--1164, 2016.

\bibitem{JMD10}
Prateek Jain, Raghu Meka, and Inderjit~S. Dhillon.
\newblock Guaranteed rank minimization via singular value projection.
\newblock In {\em NIPS}, 2010.

\bibitem{jain2013low}
Prateek Jain, Praneeth Netrapalli, and Sujay Sanghavi.
\newblock Low-rank matrix completion using alternating minimization.
\newblock In {\em STOC}, 2013.

\bibitem{jin2016provable}
Chi Jin, Sham~M Kakade, and Praneeth Netrapalli.
\newblock Provable efficient online matrix completion via non-convex stochastic
  gradient descent.
\newblock In {\em NIPS}, 2016.

\bibitem{KapralovT13}
Michael Kapralov and Kunal Talwar.
\newblock On differentially private low rank approximation.
\newblock In {\em SODA}, 2013.

\bibitem{KS08}
Shiva~Prasad Kasiviswanathan and Adam Smith.
\newblock A note on differential privacy: Defining resistance to arbitrary side
  information.
\newblock {\em CoRR}, arXiv:0803.39461 [cs.CR], 2008.

\bibitem{kearns2014mechanism}
Michael Kearns, Mallesh Pai, Aaron Roth, and Jonathan Ullman.
\newblock Mechanism design in large games: Incentives and privacy.
\newblock In {\em ITCS}, 2014.

\bibitem{keshavan2010matrix}
Raghunandan~H Keshavan, Andrea Montanari, and Sewoong Oh.
\newblock Matrix completion from a few entries.
\newblock {\em IEEE Transactions on Information Theory}, 2010.

\bibitem{KorenB15}
Yehuda Koren and Robert~M. Bell.
\newblock Advances in collaborative filtering.
\newblock In {\em Recommender Systems Handbook}, pages 77--118. Springer US,
  2015.

\bibitem{korolova2010privacy}
Aleksandra Korolova.
\newblock Privacy violations using microtargeted ads: A case study.
\newblock In {\em 2010 IEEE International Conference on Data Mining Workshops}.
  IEEE, 2010.

\bibitem{ialm}
Zhouchen Lin, Minming Chen, and Yi~Ma.
\newblock The augmented lagrange multiplier method for exact recovery of
  corrupted low-rank matrices.
\newblock {\em CoRR}, abs/1009.5055, 2010.

\bibitem{liu2015fast}
Ziqi Liu, Yu-Xiang Wang, and Alexander Smola.
\newblock Fast differentially private matrix factorization.
\newblock In {\em Proceedings of the 9th ACM Conference on Recommender
  Systems}, 2015.

\bibitem{apple}
Robert McMillan.
\newblock Apple tries to peek at user habits without violating privacy.
\newblock {\em The Wall Street Journal}, 2016.

\bibitem{MM09}
F.~McSherry and I.~Mironov.
\newblock {Differentially private recommender systems: building privacy into
  the net}.
\newblock In {\em Symp. Knowledge Discovery and Datamining (KDD)}, pages
  627--636. ACM New York, NY, USA, 2009.

\bibitem{NS10}
Arvind Narayanan and Vitaly Shmatikov.
\newblock Myths and fallacies of ``personally identifiable information''.
\newblock {\em Commun. ACM}, 53(6):24--26, 2010.

\bibitem{recht2011simpler}
Benjamin Recht.
\newblock A simpler approach to matrix completion.
\newblock {\em Journal of Machine Learning Research}, 2011.

\bibitem{shalev2011large}
Shai Shalev-shwartz, Alon Gonen, and Ohad Shamir.
\newblock Large-scale convex minimization with a low-rank constraint.
\newblock In Lise Getoor and Tobias Scheffer, editors, {\em Proceedings of the
  28th International Conference on Machine Learning (ICML-11)}, pages 329--336,
  New York, NY, USA, 2011. ACM.

\bibitem{shamir2011collaborative}
Ohad Shamir and Shai Shalev-Shwartz.
\newblock Collaborative filtering with the trace norm: Learning, bounding, and
  transducing.
\newblock In {\em COLT}, 2011.

\bibitem{srebro2005rank}
Nathan Srebro and Adi Shraibman.
\newblock Rank, trace-norm and max-norm.
\newblock In {\em International Conference on Computational Learning Theory},
  2005.

\bibitem{TTZ15}
Kunal Talwar, Abhradeep Thakurta, and Li~Zhang.
\newblock Nearly optimal private lasso.
\newblock In {\em NIPS}, 2015.

\bibitem{tao2012topics}
Terence Tao.
\newblock {\em Topics in random matrix theory}, volume 132.
\newblock American Mathematical Society, 2012.

\bibitem{tewari2011greedy}
Ambuj Tewari, Pradeep~K Ravikumar, and Inderjit~S Dhillon.
\newblock Greedy algorithms for structurally constrained high dimensional
  problems.
\newblock In {\em NIPS}, 2011.

\bibitem{WFS15}
Yu-Xiang Wang, Stephen Fienberg, and Alex Smola.
\newblock Privacy for free: Posterior sampling and stochastic gradient monte
  carlo.
\newblock In {\em Proceedings of the 32nd International Conference on Machine
  Learning (ICML-15)}, 2015.

\bibitem{YahooMov}
Yahoo.
\newblock C15 - yahoo! music user ratings of musical tracks, albums, artists
  and genres, version 1.0.
\newblock {\em Webscope}, 2011.

\bibitem{inductivemc}
Hsiang-Fu Yu, Prateek Jain, Purushottam Kar, and Inderjit Dhillon.
\newblock Large-scale multi-label learning with missing labels.
\newblock In {\em ICML}, 2014.

\end{thebibliography}
\ifnum\supp=1
\appendix
\section{Frank-Wolfe algorithm}
\label{sec:FW}

We use the classic Frank-Wolfe algorithm \cite{FW56} as one of the optimization building blocks for our differentially private algorithms. 
In Algorithm \ref{Algo:FW},  we state the Frank-Wolfe method to solve the following convex  optimization problem:
\begin{equation}
	\hY=\arg\min\limits_{\nuc{Y}\leq k}\frac{1}{2|\Omega|}\lfrob{\pomega\left(Y-\mat{Y}^*\right)}^2.
	\label{eq:matCompletion}
\end{equation}
In this paper, we use the approximate version of the algorithm from \cite{Jag13}. The only difference is that, instead of using an exact minimizer to the linear optimization problem, Line \ref{line:abcdanon}
of Algorithm \ref{Algo:FW} uses an oracle that minimizes the problem up to a slack of $\gamma$. In the following, we provide the convergence guarantee for Algorithm \ref{Algo:FW}.

\begin{algorithm}[tb]
\caption{Approximate Frank-Wolfe algorithm}
\label{Algo:FW}
\begin{algorithmic}
\STATE {\bfseries Input:} Set of revealed entries: $\Omega$, operator: $\pomega$, matrix: $\pomega(Y^*)\in\re^{m\times n}$, nuclear norm constraint: $k$, time bound: $T$, slack parameter: $\gamma$ 
\STATE $\matT{Y}{0}\leftarrow \{0\}^{m\times n}$
\FOR{$t\in[T]$}
\STATE $\matT{W}{t-1}\leftarrow\frac{1}{|\Omega|}\pomega\left(\matT{Y}{t-1}-Y^*\right)$
\STATE Obtain $\matT{Z}{t-1}$ with $\nuc{\matT{Z}{t-1}}\leq k$ s.t.   {$\left(\ip{\matT{W}{t-1}}{\matT{Z}{t-1}}-\min\limits_{\nuc{\Theta}\leq k}\ip{\matT{W}{t-1}}{\Theta}\right)\leq \gamma$} \label{line:abcdanon} 
\STATE $\matT{Y}{t}\leftarrow\left(1-\frac{1}{T}\right)\matT{Y}{t-1}+\frac{\matT{Z}{t-1}}{T}$ \label{line:update12}
\ENDFOR
\STATE Return $\matT{Y}{T}$
\end{algorithmic}
\end{algorithm}

\mypar{Note} Observe that the algorithm converges at the rate of $O(1/T)$ even with an error slack of $\gamma$. While such a convergence rate is sufficient for us to prove our utility guarantees, we observe that this rate is rather slow in practice. 

\begin{theorem}[Utility guarantee]
	Let $\gamma$ be the slack in the linear optimization oracle in Line \ref{line:abcdanon}
	of Algorithm \ref{Algo:FW}. Then, following is true for $\matT{Y}{T}$:
	{
	\begin{align*}
	 & \widehat{F}\left(\matT{\Ypriv}{T}\right)   -\min\limits_{\nuc{Y}\leq k}\widehat{F}\left(Y\right) \leq \frac{k^2}{|\Omega|T}+\gamma.
	\end{align*}
	}
	\label{thm:utilFW}
\end{theorem}

\begin{proof}[Proof (Adapted from \cite{Jag13})]
	Let $\mathcal{D}\in\re^{m\times n}$ some fixed domain. We will define the curvature parameter $C_f$ of any differentiable function $f:\mathcal{D}\to\re$ to be the following:
	\begin{equation*}
	C_f=\max\limits_{\substack{x,s\in\mathcal{D}, \mu\in[0,1]:\\  y=x+\mu(s-x)}}\frac{2}{\mu^2}\left(f(y)-f(x)-\ip{y-x}{\grad f(x)}\right).
	\label{eq:curvature}
	\end{equation*}
	
	\sloppy In the optimization problem in \eqref{eq:matCompletion}, let $f(Y)=\frac{1}{2|\Omega|}\frob{\pomega\left(Y-Y^*\right)}^2$, and $\matT{G}{t-1}=\arg\min\limits_{\nuc{\Theta}\leq k}\ip{\matT{W}{t-1}}{\Theta}$, where $\matT{W}{t-1}$ is as defined in Line 3 of Algorithm \ref{Algo:FW}. We now have the following due to smoothness:
	{
	\begin{align*}
	f\left(\matT{Y}{t}\right)&=f\left(\matT{Y}{t-1}+\frac{1}{T}\left(\matT{Z}{t-1}-\matT{Y}{t-1}\right)\right)\nonumber\\
	&\leq f\left(\matT{Y}{t-1}\right)+\frac{1}{2T^2}C_f    +\frac{1}{T}\ip{\matT{Z}{t-1}-\matT{Y}{t-1}}{\grad f\left(\matT{Y}{t-1}\right)} .
	\label{eq:as123as} \numberthis
	\end{align*}}
	Now, by the $\gamma$-approximation property in Line 4 of Algorithm \ref{Algo:FW}, we have:\\ {
		\begin{align*}
		& \ip{\matT{Z}{t-1}-\matT{Y}{t-1}}{\grad f\left(\matT{Y}{t-1}\right)}     \leq \ip{\matT{G}{t-1}-\matT{Y}{t-1}}{\grad f\left(\matT{Y}{t-1}\right)}+\gamma.
		\end{align*}} Therefore, we have the following from \eqref{eq:as123as}:
	\begin{align*}
	 f\left(\matT{Y}{t}\right) & \leq f\left(\matT{Y}{t-1}\right) +\frac{C_f}{2T^2}\left(1+\frac{2T\gamma}{C_f}\right)    +\frac{1}{T}\ip{\matT{G}{t-1}-\matT{Y}{t-1}}{\grad f\left(\matT{Y}{t-1}\right)}. \numberthis
	\label{eq:abcd12a}
	\end{align*}
	Recall the definition of  $ \hY $  from \eqref{eq:matCompletion}, and let $h(\Theta)=f(\Theta)-f(\hY)$. By convexity, we have the following (also called the duality gap):
	\begin{equation}
	\ip{\matT{Y}{t}-\matT{G}{t}}{\grad f\left(\matT{Y}{t}\right)}\geq h\left(\matT{Y}{t}\right).
	\label{eq:duality}
	\end{equation}
Therefore, from \eqref{eq:abcd12a} and \eqref{eq:duality}, we have the following:
		\begin{align*}
		h\left(\matT{Y}{T}\right) &\leq h\left(\matT{Y}{T-1}\right)-\frac{h\left(\matT{Y}{T-1}\right)}{T}  +\frac{C_f}{2T^2}\left(1+\frac{2T\gamma}{C_f}\right)   \\
		&=\left(1-\frac{1}{T}\right)h\left(\matT{Y}{T-1}\right)+\frac{C_f}{2T^2}\left(1+\frac{2T\gamma}{C_f}\right)\\
		&\leq\frac{C_f}{2T^2}\left(1+\frac{2T\gamma}{C_f}\right)  \cdot  \left(1+\left(1-\frac{1}{T}\right)+\left(1-\frac{1}{T}\right)^2+\cdots\right) \\
		& \leq \frac{C_f}{2T}\left(1+\frac{2T\gamma}{C_f}\right)  =\frac{C_f}{2T}+\gamma\\
		 \Leftrightarrow f\left(\matT{Y}{T}\right)-f\left(\hY\right)& \leq \frac{C_f}{2T}+\gamma. \numberthis \label{eq:aflk123a} 
		\end{align*}
	With the above equation in hand, we bound the term $C_f$ for the stated $f(\Theta)$ to complete the proof. Notice that  $\frac{2k^2}{|\Omega|}$ is an upper bound on the curvature constant $C_f$ (See Lemma 1 from \cite{shalev2011large}, or Section 2 of \cite{clarkson2010coresets}, for a proof). Therefore, from \eqref{eq:aflk123a}, we get:
	\begin{equation*}
	f\left(\matT{Y}{T}\right)-f\left(\hY\right)\leq \frac{k^2}{|\Omega|T}+\gamma,
	\label{eq:lsdkf12}
	\end{equation*}which completes the proof.
\end{proof}
\section{Private matrix completion via singular value decomposition (SVD)}
\label{app:SVDapprox}

In this section, we study a simple SVD-based algorithm for differentially private matrix completion. Our SVD-based algorithm for matrix completion just computes a low-rank approximation of $\pomega(Y^*)$, but still provides {\em reasonable} error guarantees \cite{keshavan2010matrix}. Moreover, the algorithm forms a foundation for more sophisticated algorithms like alternating minimization \cite{HW14}, singular value projection \cite{JMD10} and singular value thresholding \cite{CCS10}. Thus, similar ideas may be used to extend our approach.

\mypar{Algorithmic idea} At a high level, given rank $r$,  Algorithm \ref{Algo:PrivSVD} first computes a differentially private version of the top-$r$ right singular subspace of $\pomega(Y^*)$, denoted by $V_r$. Each user projects her data record onto $V_r$ (after appropriate scaling) to complete her row of the matrix. Since each user's completed row depends on the other users via the global computation which is performed under differential privacy, the overall agorithm satisfies \jc. In principle, this is the same as in Section \ref{sec:privFW}, except now it is a direct rank-$r$ decomposition instead of an iterative rank-1 decomposition. Also, our overall approach is similar to that of \cite{MM09}, except that each user in \cite{MM09} uses a nearest neighbor algorithm  in the local computation phase (see Algorithm \ref{Algo:PrivSVD}). Additionally, in contrast to \cite{MM09}, we provide a formal generalization guarantee.
\begin{algorithm}
\caption{Private Matrix Completion via SVD}
\label{Algo:PrivSVD}
\begin{algorithmic}
\STATE {\bfseries Input:} Privacy parameters: $(\epsilon,\delta)$, matrix dimensions: $(m,n)$,  uniform $\ell_2$-bound on the rows of $\pomega(Y^*)$: $L$, and rank bound: $r$ 
\STATE {\bf Global computation: }  Compute the top-$r$ subspace $V_r$ for the matrix ${\widehat{W}}\leftarrow \sum\limits_{i=1}^m W_i+ N$, where {$W_i=\Pi_{L}\left(\pomega\left({Y^*_i}\right)\right)^\top\Pi_{L}\left(\pomega\left({Y^*_i}\right)\right)$}, $\Pi_{L}$ is the projection onto the $\ell_2$-ball of radius $L$, {$N\in\re^{n\times n}$} corresponds to a matrix with i.i.d. entries from {$\mathcal{N}(0,\sigma^2)$}, and \mbox{$\sigma\leftarrow L^2 \sqrt{64\log(1/\delta)}/\epsilon$}
\STATE {\bf Local computation:} Each user $i$ computes the $i$-th row of the private approximation $\widehat{Y}$: \mbox{$\widehat{Y}_i\leftarrow \frac{mn}{|\Omega|}\pomega\left(Y^*_i\right)V_r {V_r}^\top$}
\end{algorithmic}
\end{algorithm}

\subsection{Privacy and utility analysis}
\label{sec:privUtilSVD}

We now present the privacy and generalization guarantees for the above algorithm. 

\begin{theorem}
	Algorithm \ref{Algo:PrivSVD} satisfies $(\epsilon,\delta)$-\jc.
	\label{thm:privSVD}
\end{theorem}
The proof of privacy for Algorithm \ref{Algo:PrivSVD} follows immediately from the proof of Theorem \ref{thm:priv}, as the key step of computing the top eigenvectors of the $W$ matrix remains the same. 

For the generalization error bound for Algorithm \ref{Algo:PrivSVD}, we use the standard low-rank matrix completion setting, i.e., entries are sampled i.i.d., and the underlying matrix $Y^*$ is incoherent \ref{def:incoherence}. Intuitively, incoherence ensures that the left and right singular subspaces of a matrix have a low correlation with the standard basis. The scale of $\mu$ is $[0,\max\{m,n\}]$. Since, we are assuming $m\geq n$ throughout the paper, $\mu\in[0,m]$.
\begin{definition}[$\mu$-incoherence \cite{jin2016provable}]
	Let $Y\in\re^{m\times n}$ be a matrix of rank at most $r$, and let $U\in\re^{m\times r}$ and $V\in\re^{r\times n}$ be the left and right singular subspaces of $Y$. Then, the incoherence $\mu$ is the following:
	{$$\mu=\max\left\{\frac{m}{r}\max\limits_{1\leq i\leq m}\ltwo{UU^Te_i},\frac{n}{r}\max\limits_{1\leq i\leq n}\ltwo{VV^Tf_i}\right\}.$$}
	Here, $e_i\in\re^{m}$ and $f_i\in\re^{n}$ are the $i$-th standard basis vectors in $m$ and $n$ dimensions, respectively.
	\label{def:incoherence}
\end{definition}

Under the above set of assumptions, we get: 
\begin{theorem}
	Let $Y^*\in \re^{m\times n}$ be a rank-$r$, $\mu$-incoherent matrix with condition number $\kappa=\ltwo{Y^*}/\lambda_r(Y^*)$, where $\lambda_r(\cdot)$ corresponds to the $r$-th largest singular value. Also, let the set of known entries $\Omega$ be sampled uniformly at random s.t. $|\Omega|\geq c_0 \kappa^2 \mu  m r  \log m$ for a large constant $c_0>0$. Let $\ltwo{\pomega(Y^*)_i}\leq L$ for every row $i$ of $Y^*$. Then, with probability at least $2/3$ over the outcomes of the algorithm, the following holds for $\Yp$ estimated by Algorithm~\ref{Algo:PrivSVD}: 
	{ $$\widehat{F}(\Yp)=O\left(\frac{L^4\kappa^4 m^3n^4\cdot r\cdot \Delta^2_{\epsilon,\delta}}{|\Omega|^4 \|Y^*\|_2^2} + {\frac{\mu \|Y^*\|_2^2\cdot r^2 \log m}{n\cdot |\Omega|}}\right),$$}
	where the privacy parameter is $\Delta_{\epsilon,\delta}=\sqrt{64\log(1/\delta)}/\epsilon$. 
	
	Using $L\leq \ltwo{Y^*}$, we get: 
		{ \begin{align*}
		 F(\Yp) & = O\left(\frac{\min\left(L^2, \frac{\mu^2 r n}{m}\right)\kappa^4 m^3n^4\cdot r\cdot \Delta^2_{\epsilon,\delta}}{|\Omega|^4 }  + {\frac{\mu \|Y^*\|_2^2\cdot r^2 \log m}{n\cdot |\Omega|}}\right).
		\end{align*}}
		The $O\left(\cdot\right)$ hides only universal constants.
	\label{thm:svd}
\end{theorem}

For a proof of this theorem, see Section \ref{sec:proofSVD}.

\begin{remark}
	Let $Y^*$ be a rank one incoherent matrix with $Y^*_{ij}=\Theta(1)$, $|\Omega|=m\sqrt{n}$, $L=O(n^{1/4})$, and $\mu=O(1)$. Notice that the spectral norm $\ltwo{Y^*}\approx \sqrt{mn}$. Hence, the first term in the bound reduces to $O\left(\frac{n^2}{m^2}\right)$ and the second error term is $O\left(\frac{1}{\sqrt{n}}\right)$ , whereas a trivial solution of $Y=0$ leads to $O(1)$ error. Similar to the behavior in Remark~\ref{rem:ge1}, the first term above increases with $n$, and decreases with increasing $m$ due to the noise added, while the second term decreases with increasing $n$ due to more sharing between users. 
	\label{rem:svd1}
\end{remark}

\begin{remark}
	Under the assumptions of Theorem~\ref{thm:svd}, the second term can be arbitrarily small for other standard matrix completion methods like the FW-based method (Algorithm~\ref{Algo:PrivFW}) studied in Section \ref{sec:privFW} above. However, the first error term for such methods can be significantly larger. For example, the error of Algorithm~\ref{Algo:PrivFW} in the setting of Remark \ref{rem:svd1}  is $\approx O\left(\frac{n^{13/24}}{m^{5/12}}\right)$ as the second term in Corollary \ref{cor:abcfd12} vanishes in this setting; in contrast, the error of the SVD-based method (Algorithm~\ref{Algo:PrivSVD}) is $O\left(\frac{n^2}{m^2}+\frac{1}{\sqrt{n}}\right)$. On the other hand, if the data does not satisfy the assumptions of Theorem~\ref{thm:svd}, then the error incurred by Algorithm~\ref{Algo:PrivSVD} can be significantly larger (or even trivial) when compared to that of Algorithm~\ref{Algo:PrivFW}.
\end{remark}

\subsubsection{Proof of Theorem \ref{thm:svd}}
\label{sec:proofSVD}
\begin{proof}
	Let $B=\frac{1}{p} \pomega(Y^*)$ where $p=|\Omega|/(m\cdot n)$ and let $V_r$ be the top-$r$ right singular subspace of $B$. Suppose $\Pi_r=V_r V_r^\top$ be the projector onto that subspace. Recall that $\vp_r$ is the right singular subspace defined in Algorithm \ref{Algo:PrivSVD} and let $\Pip_r=\vp_r \vp_r^T$ be the corresponding projection matrix. 
	
Then, using the triangular inequality, we have: 
{
		\begin{align*}
		\|B\Pip_r-Y^*\|_2 & \leq \|B\Pi_r-Y^*\|_2 + \|B\Pip_r-B\Pi_r\|_2  \\
		&  \leq c_1 {\|Y^*\|_2}\sqrt{\frac{\mu  m  r \log m}{|\Omega|}}  + \|B\Pip_r-B\Pi_r\|_2 \numberthis \label{eq:svd1},
		\end{align*}
}where the second inequality follows from the following standard result (Lemma \ref{lem:mcnorm}) from the matrix completion literature, and holds w.p. $\geq 1-1/m^{10}$. 
	\begin{lemma}[Follows from Lemma A.3 in \cite{jin2016provable}]\label{lem:mcnorm}
		Let $M$ be an $m \times n$ matrix with $m \geq n$, rank $r$, and incoherence $\mu$, and $\Omega$ be a subset of i.i.d. samples from $M$.  There exists universal constants $c_1$ and $c_0$ such that if $|\Omega| \geq c_0 \mu m r \log m$, then with probability at least $1 - 1/m^{10}$, we have: $$\ltwo{M - \frac{mn}{|\Omega|}\pomega (M)}\leq c_1 \ltwo{M}\sqrt{\frac{\mu\cdot m\cdot r \log m}{|\Omega|}}.$$
	\end{lemma}
	Using Theorem 6 of \cite{dwork2014analyze}, the following holds with probability at least $2/3$,  \begin{equation}\label{eq:svd2}\ltwo{\Pip_r-\Pi_r}=O\left(\frac{L^2\sqrt{n}\Delta_{\epsilon,\delta}}{\alpha^2_r-\alpha_{r+1}^2}\right), \end{equation}
	where $\alpha_i$ is the $i$-th singular value of $\pomega(Y^*)=p \cdot B$. 
	
	Recall that $\kappa=\ltwo{Y^*}/\lambda_r(Y^*)$, where $\lambda_r$ is the $r$-th singular value of $Y^*$. Let  $|\Omega|\geq c_0 \kappa^2 \mu m r \log m$ with a  large  constant $c_0 > 0$. Then, using Lemma~\ref{lem:mcnorm} and Weyl's inequality, we have (w.p. $\geq 1-1/m^{10}$): 
	{
		\begin{align*}
		& \alpha_{r}\geq 0.9 \cdot p \frac{1}{\kappa} \|Y^*\|_2, \qquad \text{and} \qquad   \alpha_{r+1}\leq c_1 p \cdot {\|Y^*\|_2}\sqrt{\frac{\mu \cdot m \cdot r \log m}{|\Omega|}}\leq 0.1 \cdot \alpha_r  \numberthis\label{eq:svd3}
		\end{align*}}
	Similarly, \begin{equation}\label{eq:svd4} \|B\|_2 \leq 2 \|Y^*\|_2, w.p. \geq 1-1/m^{10}.\end{equation}
	Using \eqref{eq:svd1}, \eqref{eq:svd2}, \eqref{eq:svd3}, and \eqref{eq:svd4}, we have w.p. $\geq \frac{2}{3}-\frac{5}{m^{10}}$: 
		\begin{align*}
		\|B\Pip_r-Y^*\|_2&\leq 8\|Y^*\|_2 \cdot \frac{L^2\kappa^2\sqrt{n}\Delta_{\epsilon,\delta}}{p^2 \|Y^*\|_2^2}    + c_1 {\|Y^*\|_2}\sqrt{\frac{\mu \cdot m \cdot r \log m}{|\Omega|}} .
		\end{align*}
	Recall that $\Yp=\frac{1}{p} \pomega(Y^*)\Pip_r=B\Pip_r$. Hence: 
		\begin{align*}
		\frac{\|B\Pip_r-Y^*\|_2^2}{mn} & \leq O\left(\frac{L^4\kappa^4 n\Delta^2_{\epsilon,\delta}}{mn\cdot p^4 \|Y^*\|_2^2} \right) + c_1 {\|Y^*\|_2^2}{\frac{\mu \cdot m \cdot r \log m}{mn\cdot |\Omega|}} .
		\end{align*}
		The theorem now follows by using $\|A\|_F^2\leq r \|A\|_2^2$, where $r$ is the rank of $A$. 
\end{proof}

\section{Additional experimental evaluation}


\label{app:exptVal}
Here, we provide the empirical results for our private Frank-Wolfe algorithm (Algorithm~\ref{Algo:PrivFW}) as well as the `SVD after cleansing method' of \cite{MM09} for the following additional datasets:
\begin{enumerate}[\itemsep=0pt]
\item \emph{Synthetic-900:} We generate a random rank-one matrix $Y^*=uv^T$ with  unit $\ell_\infty$-norm, $m=500$K, and $n=900$.
\item \emph{MovieLens10M (Top 900):}  We pick the $n=900$ most rated movies from the Movielens10M dataset, which has $m\approx 70$K users of the $\approx 71$K users in the dataset. 
\item \emph{Netflix (Top 900):} We pick the $n=900$ most rated movies from the Netflix prize dataset, which has $m\approx477$K users of the $\approx480$K users in the dataset.
\item \emph{Yahoo! Music (Top 900):}  We pick the $n=900$ most rated songs from the Yahoo! music dataset, which has $m \approx 998$K users of the $\approx 1$M users in the dataset. We rescale the ratings to be from 0 to 5.
\end{enumerate}

\begin{figure}[ht]
	\centering
	\begin{tabular}{ccc}
		\hspace*{-15pt}
	\begin{minipage}[b]{0.33\textwidth}
		\includegraphics[width=\textwidth]{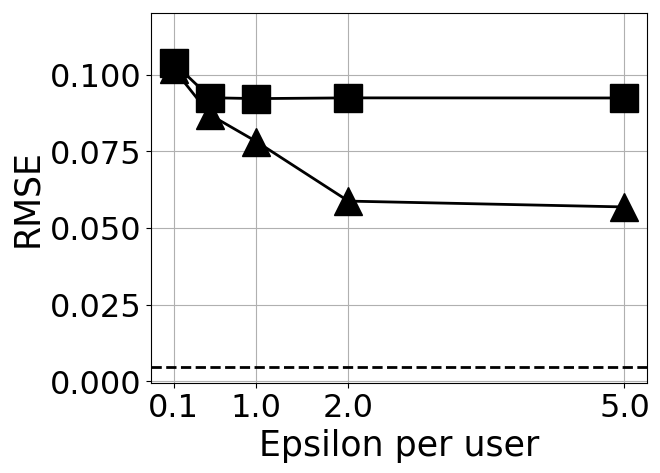}
	\end{minipage}
	&		\hspace*{-15pt}
	\begin{minipage}[b]{0.33\textwidth}
		\includegraphics[width=\textwidth]{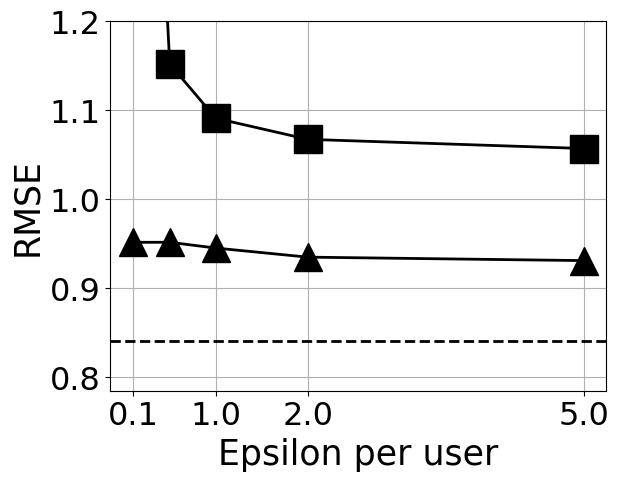}
	\end{minipage}
	& \hspace*{-15pt}
	\begin{minipage}[b]{0.33\textwidth}
		\includegraphics[width=\textwidth]{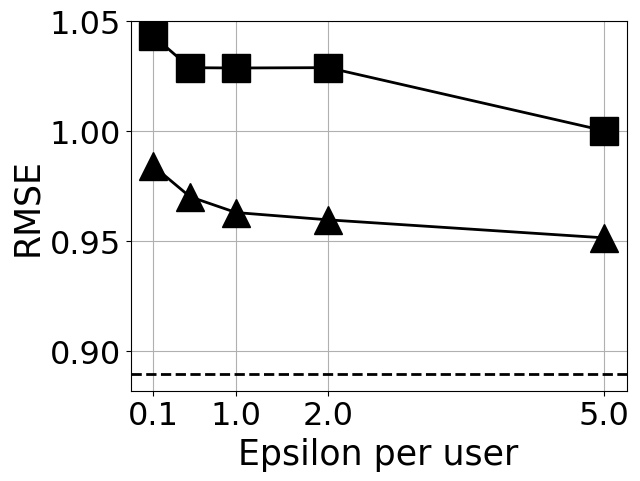}
	\end{minipage}\\
(a)&(b)&(c)\\
\hspace*{-15pt}
	\begin{minipage}[b]{0.33\textwidth}
	\includegraphics[width=\textwidth]{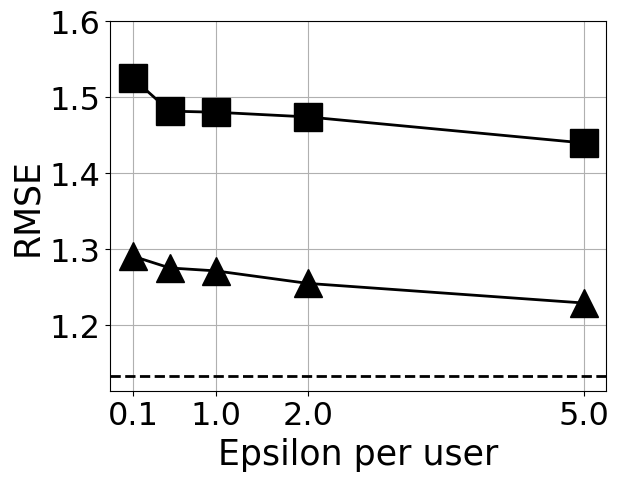}
\end{minipage}
& \hspace*{-15pt}
\begin{minipage}[b]{0.33\textwidth}
	\includegraphics[width=\textwidth]{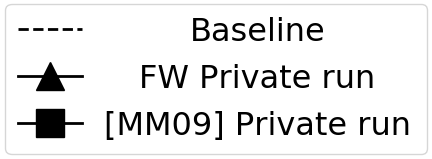}
\end{minipage}
&\\
(d)&(e)
\end{tabular}
	\caption{Root mean squared error (RMSE) vs. $\eps$, on (a) Synthetic-900, (b) MovieLens10M, (c) Netflix, and (d) Yahoo! Music datasets, for $\delta = 10^{-6}$. A legend for all the plots is given in (e).}
	\label{fig:syn900}
\end{figure}

We follow the same experimental procedure as in Section~\ref{sec:exptVal}. For each dataset, we cross-validate over the nuclear norm bound $k$, and the number of iterations $T$. For $k$, we set it to the actual nuclear norm for Synthetic-900 dataset, and choose from  $\{ 150000, 160000\}$ for Netflix,  $\{50000, 60000\}$ for MovieLens10M, and $\{260000, 270000\}$ for the Yahoo! Music dataset.  We choose $T$ from various values in $[5,50]$.

In Figure \ref{fig:syn900}, we show the results of our experiments on the Synthetic-900 dataset in plot (a), MovieLens10M (Top 900) in plot (b), Netflix (Top 900) in plot (c), and Yahoo! Music (Top 900) in plot (d). In all the plots, we see that the test RMSE for private Frank-Wolfe almost always incurs a significantly lower error than the method of \cite{MM09}.

\section{Omitted proofs and existing results}
\label{sec:omittedProofs}
In this section, we provide detailed proofs, and state the used existing results that have been omitted from the main body of the paper.

\subsection{Proofs of privacy and utility for Private Frank-Wolfe (Algorithm~\ref{Algo:PrivFW})}
\label{app:privFW}

\subsubsection{Proof of privacy}
\label{app:privFWpriv}

\begin{proof}[Proof of Theorem \ref{thm:priv}]
Note that we require $\eps > 2\log{\left(\frac{1}{\delta}\right)}$ for input parameters $(\eps, \delta)$ in Algorithm~\ref{Algo:PrivFW}. Assuming this is satisfied, let us consider the sequence of matrices $\matT{W}{1},\cdots,\matT{W}{T}$ produced by function $\mathcal{A}_{\sf global}$. Notice that if every user $i \in [m]$ knows this sequence, then she can construct her updates $\matT{Y_i}{1},\cdots,\matT{Y_i}{T}$ by herself \emph{independent} of any other user's data. Therefore, by the \emph{post-processing} property of differential privacy \cite{DMNS,DR14}, it follows that as long as function $\mathcal{A}_{\sf global}$ satisfies $(\epsilon,\delta)$-differential privacy, one can ensure $(\epsilon,\delta)$-\jc for Algorithm~\ref{Algo:PrivFW}, i.e., the combined pair of functions $\mathcal{A}_{\sf global}$ and $\mathcal{A}_{\sf local}$. (Recall that the post-processing property of differential privacy states that any operation performed on the output of a differentially private algorithm, without accessing the raw data, remains differentially private with the same level of privacy.) Hence, Lemma \ref{lem:diffPr} completes the proof of privacy.
\end{proof}

\begin{lemma}
For input parameters $(\eps, \delta)$ such that $\eps \leq 2\log{\left(\frac{1}{\delta}\right)}$, let $\matT{W}{t}$ be the output in every iteration $t\in[T]$ of function $\mathcal{A}_{\sf global}$  in Algorithm \ref{Algo:PrivFW}. Then, $\mathcal{A}_{\sf global}$  is $(\epsilon,\delta)$-differentially private. \label{lem:diffPr}
\end{lemma}
\begin{proof}
We are interested in the function ${\sf Cov}(\matT{A}{t}) =\matT{A}{t}^\top \matT{A}{t},$ where $\matT{A}{t}=\pomega\left(\matT{Y}{t}-Y^*\right)$. Since $\ltwo{\pomega\left(\matT{Y}{t}\right)_i}\leq L$ and  $\ltwo{\pomega\left(Y^*\right)_i}\leq L$ for all rows $i\in[m]$, we have that the $\ell_2$-sensitivity of ${\sf Cov}(\matT{A}{t})$ is $4L^2$. Recall that the $\ell_2$-sensitivity of ${\sf Cov}$ corresponds to the maximum value of $\lfrob{{\sf Cov}(A)-{\sf Cov}(A')}$ for any two matrices $A,A'$ in the domain, and differing in exactly one row. Using the Gaussian mechanism (Propositions 6), Proposition 3, and Lemma 7 (composition property) from \cite{bun2016concentrated}, it follows that adding Gaussian noise with standard deviation $\sigma = \frac{L^2\sqrt{64\cdot T\log(1/\delta)}}{\epsilon}$ in each iteration of the global component of private Frank-Wolfe (function $\mathcal{A}_{\sf global}$) ensures $(\epsilon,\delta)$-differential privacy for $\eps \leq 2\log{(1/\delta)}$.
\end{proof}

\subsubsection{Proof of utility}
\label{app:privFWutil}

\begin{proof}[Proof of Theorem~\ref{thm:utilRransferLearning}]

Recall that in function $\mathcal{A}_{\sf global}$  of Algorithm \ref{Algo:PrivFW}, the matrix $\matT{\widehat{W}}{t}$ captures the total error covariance corresponding to all the users at a given time step $t$, i.e., $\matT{A}{t}^\top\matT{A}{t}=\sum\limits_{i\in[m]}\matT{A_i}{t}^\top\matT{A_i}{t}$. Spherical Gaussian noise of appropriate scale is added to ensure that $\matT{\widehat{W}}{t}$ is computed under the constraint of differential privacy. Let  $\mathbf{\widehat v}$ be the top eigenvector of $\matT{\widehat{W}}{t}$, and let $\widehat \lambda^2$ be the corresponding eigenvalue. In Lemma \ref{lem:singularBound}, we first show that $\widehat{\lambda}$ is a reasonable approximation to the energy of $\matT{A}{t}$ captured by $\mathbf{\widehat v}$, i.e., $\ltwo{\matT{A}{t}\mathbf{\widehat v}}$. Furthermore, in Lemma \ref{lem:DTTZ14} we show that $\mathbf{\widehat{v}}$ captures sufficient energy of the matrix $\matT{A}{t}$. Hence, we can conclude that one can use
$\mathbf{\widehat{v}}$ as a proxy for the top right singular vector of $\matT{A}{t}$.

	\begin{lemma}
		With probability at least $1-\beta$, the following is true:
		\begin{align*}
		\ltwo{\matT{A}{t}\mathbf{\widehat v}}\leq \widehat{\lambda}+O\left(\sqrt{\sigma\log(n/\beta)\sqrt{n}}\right).
		\end{align*}
		\label{lem:singularBound}
	\end{lemma}
	
	\begin{proof}
		Let $\mat{E}=\matT{\widehat{W}}{t}-{\matT{A}{t}}^\top\matT{A}{t}$, where the matrix $\matT{\widehat{W}}{t}$ is computed in iteration $t$ of the function $\mathcal{A}_{\sf global}$. We have,
		\begin{align*}
		\ltwo{\matT{A}{t}\mathbf{\widehat v}}^2&=\mathbf{\widehat v}^\top{\matT{A}{t}}^\top\matT{A}{t}\mathbf{\widehat v} \\
		&  =\mathbf{\widehat v}^\top\left({\matT{A}{t}}^\top\matT{A}{t}+\mat{E}\right)\mathbf{\widehat v}-\mathbf{\widehat v}^\top \mat{E}\mathbf{\widehat v}\\
		& \leq \widehat{\lambda}^2+\ltwo{\mat{E}}\\
		& \leq \widehat{\lambda}^2 + O\left(\sigma\log(n/\beta)\sqrt{n}\right) \text{ w.p. $ \geq 1-\beta$} \numberthis \label{eq:askj12}
		\end{align*}
		Inequality~\eqref{eq:askj12} follows from the spectral norm bound on the Gaussian matrix $E$ drawn i.i.d. from $\mathcal{N}\left(0,\sigma^2\right)$. (See Corollary 2.3.5 in \cite{tao2012topics} for a proof). The statement of the lemma follows from inequality~\eqref{eq:askj12}.
	\end{proof}

\begin{lemma}[Follows from Theorem 3 of \cite{dwork2014analyze}]
	Let $A\in\re^{m\times p}$ be  a matrix and let $\widehat{W}=A^\top A+E$, where $E\sim\mathcal{N}\left(0,\I_{p\times p}\sigma^2\right)$. Let $\mathbf{v}$ be the top right singular vector of $A$, and let $\mathbf{\widehat{v}}$ be the top eigenvector of $\widehat{W}$. The following is true with probability at least $1-\beta$:
	\begin{equation*}\ltwo{A\mathbf{\widehat{v}}}^2\geq\ltwo{A\mathbf{v}}^2-O\left(\sigma\log(n/\beta)\sqrt{n}\right).\end{equation*}\label{lem:DTTZ14}
\end{lemma}

Now, one can compactly write the update equation of $\matT{Y}{t}$ in function $\mathcal{A}_{\sf local}$ of Algorithm \ref{Algo:PrivFW} for all the users as:
\begin{equation}
\matT{\Ypriv}{t}\leftarrow\Pi_{L,\Omega}\left( \left(1-\frac{1}{T}\right)\matT{\Ypriv}{t-1}-\frac{k}{T}\mathbf{\widehat{u}}\mathbf{\widehat{v}^\top}\right),
\label{eq:FWupdate}
\end{equation}where $\mathbf{\widehat{u}}$ corresponds to the set of entries $\widehat{u}_i$ in function $\mathcal{A}_{\sf local}$ represented as a vector. Also, by Lemma \ref{lem:singularBound}, we can conclude that $\ltwo{\mathbf{\widehat{u}}}\leq 1$. Hence, $\matT{Y}{t}$ is in the set $\{Y: \nuc{Y}\leq k\}$ for all $t\in[T]$.
 
In the following, we incorporate the noisy estimation in the analysis of original Frank-Wolfe (stated in Section \ref{sec:FW}). In order to do so, we need to ensure a couple of properties: i) We need to obtain an appropriate bound on the slack parameter $\gamma$ in Algorithm \ref{Algo:FW}, and ii) we need to ensure that the projection operator $\Pi_{L,\Omega}$ in function $\mathcal{A}_{\sf local}$ does not introduce additional error. We do this via Lemma \ref{lem:lkjkh} and \ref{lem:kljah12} respectively.

\begin{lemma}
	For the noise variance $\sigma$ used in function $\mathcal{A}_{\sf global}$ of Algorithm \ref{Algo:PrivFW}, w.p. at least $1-\beta$, the slack parameter $\gamma$ in the linear optimization step of Frank-Wolfe algorithm is at most $O\left(\frac{k}{|\Omega|}\sqrt{\sigma\log(n/\beta)\sqrt{n}}\right)$.
	\label{lem:lkjkh}
\end{lemma} 

\begin{proof}
Recall that $\widehat \lambda^2$ corresponds to the maximum eigenvalue of $\matT{W}{t}$, and notice that $\matT{A}{t}$ is the scaled gradient of the loss function $\widehat{F}(\Theta)$ at $\Theta=\Pi_{L,\Omega}\left(\matT{Y}{t}\right)$. Essentially, we need to compute the difference between $\ip{\frac{1}{|\Omega|}\matT{A}{t}}{k\mathbf{u} \mathbf{v}^\top}$ and $\ip{\frac{1}{|\Omega|}\matT{A}{t}}{k\mathbf{\widehat u} \mathbf{\widehat v}^\top}$. Let $\alpha=\ip{\frac{1}{|\Omega|}\matT{A}{t}}{k\mathbf{u} \mathbf{v}^\top}$, and $\halpha=\ip{\frac{1}{|\Omega|}\matT{A}{t}}{k\mathbf{\widehat u} \mathbf{\widehat v}^\top}$. Now, we have the following w.p. at least $1-\beta$:
{
\begin{align*}
\halpha&=\frac{k\mathbf{\widehat v}^\top \matT{A}{t}^\top\mathbf{\widehat u} }{|\Omega|} =\frac{k\mathbf{\widehat v}^\top\matT{A}{t}^\top\matT{A}{t}\mathbf{\widehat v}}{|\Omega|\left({\widehat \lambda}+\Theta\left(\sqrt{\sigma\log(n/\beta)\sqrt{n}}\right)\right)} \\
& =\frac{k\ltwo{\matT{A}{t}\mathbf{\widehat{v}}}^2}{|\Omega|\left({\widehat \lambda}+\Theta\left(\sqrt{\sigma\log(n/\beta)\sqrt{n}}\right)\right)}\\
&\geq\frac{k\left(\ltwo{\matT{A}{t}\mathbf{v}}^2-O\left(\sigma\log(n/\beta)\sqrt{n}\right)\right)}{|\Omega|\left({\widehat \lambda}+\Theta\left(\sqrt{\sigma\log(n/\beta)\sqrt{n}}\right)\right)}\\
&=\frac{k\left(\frac{|\Omega|{\lambda}}{k}\alpha-O\left(\sigma\log(n/\beta)\sqrt{n}\right)\right)}{|\Omega|\left({\widehat \lambda}+\Theta\left(\sqrt{\sigma\log(n/\beta)\sqrt{n}}\right)\right)},\numberthis\label{eq:andj13}
\end{align*}}where ${\lambda}^2$ is the maximum eigenvalue of $\matT{A}{t}^\top \matT{A}{t}$, the second equality follows from the definition of $\mathbf{\widehat u}$, and the inequality follows from Lemma \ref{lem:DTTZ14}.  One can rewrite \eqref{eq:andj13} as:
{
\begin{align*}
\alpha-\halpha &\leq \underbrace{\left(1-\frac{{ \lambda}}{\left({\widehat \lambda} +\Theta\left(\sqrt{\sigma\log(n/\beta)\sqrt{n}}\right)\right)}\right)\alpha}_{E_1} + \underbrace{O\left(\frac{k\sigma\log(n/\beta)\sqrt{n}}{|\Omega|\left({\widehat \lambda}+\Theta\left(\sqrt{\sigma\log(n/\beta)\sqrt{n}}\right)\right)}\right)}_{E_2}.
\numberthis \label{eq:abs123x}
\end{align*}}
We will analyze $E_1$ and $E_2$ in \eqref{eq:abs123x} separately. One can write $E_1$ in \eqref{eq:abs123x} as follows:
{
\begin{align*}
E_1&=\left(\frac{\left({\widehat \lambda}+O\left(\sqrt{\sigma\log(n/\beta)\sqrt{n}}\right)\right)-\lambda}{\left({\widehat \lambda}+\Theta\left(\sqrt{\sigma\log(n/\beta)\sqrt{n}}\right)\right)}\right)\alpha  \\
& = \frac{k}{|\Omega|}\left(\frac{\left({\widehat \lambda}+O\left(\sqrt{\sigma\log(n/\beta)\sqrt{n}}\right)\right)-\lambda}{\left({\widehat \lambda}+\Theta\left(\sqrt{\sigma\log(n/\beta)\sqrt{n}}\right)\right)}\right)\lambda.
\label{eq:askjh123sd} \numberthis
\end{align*}
}
\sloppy By Weyl's inequality for eigenvalues, and the fact that w.p. at least $1-\beta$, we have $\ltwo{\matT{\widehat{W}}{t}-\matT{A}{t}^\top\matT{A}{t}}= O\left(\sigma\log(n/\beta)\sqrt{n}\right)$ because of spectral properties of random Gaussian matrices (Corollary 2.3.5 in \cite{tao2012topics}),  it follows that $\left|{\widehat \lambda}-\lambda\right|=O\left(\sqrt{\sigma\log(n/\beta)\sqrt{n}}\right)$. Therefore, one can conclude from \eqref{eq:askjh123sd} that $E_1=O\left(\frac{k}{|\Omega|}\sqrt{\sigma\log(n/\beta)\sqrt{n}}\right)$. Now, we will bound the term $E_2$ in \eqref{eq:abs123x}. Since ${\widehat \lambda}\geq 0$, it follows that $E_2=O\left(\frac{k}{|\Omega|}\sqrt{\sigma\log(n/\beta)\sqrt{n}}\right).$ Therefore, the slack parameter $\alpha-\hat{\alpha}=E_1+E_2=O\left(\frac{k}{|\Omega|}\sqrt{\sigma\log(n/\beta)\sqrt{n}}\right)$.
\end{proof}

\begin{lemma}
	Define the operators $\pomega$ and $\Pi_{L,\Omega}$  as described in function $\mathcal{A}_{\sf local}$ in Section~\ref{sec:privFW}. Let $f(Y)=\frac{1}{2|\Omega|}\lfrob{\pomega(Y-Y^*)}^2$ for any matrix $Y\in\re^{m\times n}$. The following is true for all  $Y\in\re^{m\times n}$: $f\left(\Pi_{L,\Omega}\left(Y\right)\right)\leq f\left(\pomega\left(Y\right)\right)$.
	\label{lem:kljah12}
\end{lemma}
\begin{proof}	
First, notice that for any matrix $M=\left[m^\top_1,\cdots,m^\top_m\right]$ (where $m^\top_i$ corresponds to the $i$-th row of $M$), $\lfrob{M}^2=\sum\limits_i\ltwo{m_i}^2$. Let $\Pi_L$ be the $\ell_2$ projector onto a ball of radius $L$, and $\mathbb{B}^{n}_L$ be a ball of radius $L$ in $n$-dimensions, centered at the origin. Then, for any pair of vectors, $v_1\in\re^n$ and $v_2\in\mathbb{B}^{n}_L$, $\ltwo{\Pi_L\left(v_1\right)-v_2}\leq \ltwo{v_1-v_2}$. This follows from the contraction property of $\ell_2$-projection. Hence, by the above two properties, and the fact that each row of the matrix $\pomega\left(Y^*\right)\in\mathbb{B}^{n}_L$, we can conclude $f\left(\Pi_{L}(\pomega(Y))\right)\leq f\left(\pomega(Y)\right)$  for any $Y\in\re^{m\times n}$. This concludes the proof.
\end{proof}
This means we can still use Theorem \ref{thm:utilFW}.  Hence, we can conclude that, w.p. $\geq 1-\beta$:
{
	\begin{align*}
	&\widehat{F}\left(\matT{\Ypriv}{T}\right) =O\left(\frac{k^2}{|\Omega|T}+\frac{k}{|\Omega|}\sqrt{\sigma\log(n/\beta)\sqrt{n}}\right) 
	\end{align*}}
Here we used the fact that the curvature parameter $C_f$ from Theorem \ref{thm:utilFW} is at most $k^2/|\Omega|$ (see \cite{jaggi2010simple} for a proof). Setting $\beta=1/3$ completes the proof. 
\end{proof}




\subsection{Result from \cite{srebro2005rank}}
Here, we provide the result from \cite{srebro2005rank} that we use for obtaining the bound in Corollary~\ref{cor:abcfd12}.
\label{app:res}

\begin{theorem}
	Let $Y^*$ be a hidden matrix, and the data samples in $\Omega$ be drawn uniformly at random from $[m]\times[n]$. Let $A\in\re^{m\times n}$ be a matrix with $\rank(A)\leq r$, and let each entry of $A$ be bounded by a constant. Then, the following holds with probability at least 2/3 over choosing $\Omega$:
	{\small$$\left|F(A)-\widehat{F}(A)\right|=\tilde O\left(\sqrt\frac{r\cdot(m+n)}{|\Omega|}\right).$$}
		The $\tilde O\left(\cdot\right)$ hides poly-logarithmic terms in $m$ and $n$.
	\label{thm:genNonPriv}
\end{theorem}

\section{Omitted pseudocode for Private Projected Gradient Descent}
\label{app:Pseudo}

\begin{algorithm}[htb]
\caption{Private Projected Gradient Descent}
\label{Algo:PGD}
\begin{algorithmic}
\STATE {\bfseries Input:} Set of revealed entries: $\Omega$, operator: $\pomega$, matrix: $\pomega(Y^*)\in\re^{m\times n}$, bound on $\ltwo{\pomega(Y^*_i)} \text{: } L$, nuclear norm constraint: $k$, time bound: $T$, step size schedule: $\eta_t$ for $t \in [T]$, privacy parameters: $(\eps, \delta)$  
\STATE $\sigma\leftarrow L^2 \sqrt{64\cdot T\log(1/\delta)}/\epsilon$
\STATE $\matT{Y}{0}\leftarrow \{0\}^{m\times n}$
\FOR{$t\in[T]$}
\STATE $\matT{Y}{t}\leftarrow \matT{Y}{t-1} - \eta_t \cdot \pomega\left(Y^* - \matT{Y}{t} \right)$
\STATE $\matT{W}{t}\leftarrow \matT{Y}{t}^\top \matT{Y}{t} + \matT{N}{t}$,  where {$\matT{N}{t}\in\re^{n\times n}$} corresponds to a matrix with i.i.d. entries from {$\mathcal{N}(0,\sigma^2)$}
\STATE $\widehat{V} \leftarrow$ Eigenvectors of $\matT{W}{t}$, $\widehat{\Lambda}^2 \leftarrow$ Diagonal matrix containing the $n$ eigenvalues of $\matT{W}{t}$
\STATE $\widehat{U} \leftarrow \matT{Y}{t} \widehat{V} \widehat{\Lambda}^{-1}$
\IF{$\sum\limits_{i \in [n]} \widehat{\Lambda}_{i,i} > k$}
\STATE Find a diagonal matrix $Z$ s.t. $\sum\limits_{i \in [n]} Z_{i,i} = k$, and $\exists\tau$ s.t. $\forall i \in [n], Z_{i,i} = \max \left(0,\widehat{\Lambda}_{i,i} - \tau\right)$ 
\ELSE
\STATE  $Z \leftarrow \widehat{\Lambda}$
\ENDIF
\STATE $\matT{Y}{t}\leftarrow \widehat{U} Z \widehat{V}^\top$
\ENDFOR
\STATE Return $\matT{Y}{T}$
\end{algorithmic}
\end{algorithm}
\fi
\end{document}